\documentclass{article}

     \usepackage[preprint]{neurips_2023}

\usepackage[utf8]{inputenc} 
\usepackage[T1]{fontenc}    
\usepackage{hyperref}       
\usepackage{url}            
\usepackage{booktabs}       
\usepackage{amsfonts}       
\usepackage{nicefrac}       
\usepackage{microtype}      
\usepackage{xcolor}         
\usepackage{breakurl}

\usepackage{amsmath}
\usepackage{amssymb}
\usepackage{mathtools}
\usepackage{amsthm}
\usepackage{dsfont}

\usepackage{tikzit}

\tikzstyle{node}=[fill=white, draw=black, shape=circle, minimum size=1mm, ultra thick]
\tikzstyle{small box}=[fill=white, draw=black, shape=rectangle, minimum height=0.5cm, minimum width=0.5cm, ultra thick]
\tikzstyle{weyl}=[fill=white, draw={rgb,255: red,0; green,0; blue,109}, shape=rectangle, minimum height=0.5cm, minimum width=0.5cm]
\tikzstyle{filled_node}=[fill=black, draw=black, shape=circle, minimum size=1mm, ultra thick]
\tikzstyle{large box}=[fill=white, draw=black, shape=rectangle, minimum height=0.5cm, minimum width=1cm, ultra thick]

\tikzstyle{thick}=[-, ultra thick]
\tikzstyle{blue_thick}=[-, ultra thick, draw=blue]
\tikzstyle{dashes}=[-, dashed, draw={rgb,255: red,191; green,191; blue,191}, dash pattern=on 2mm off 1mm, fill={rgb,255: red,244; green,228; blue,0}]
\tikzstyle{thick_arrow}=[ultra thick, ->]
\tikzstyle{dash_1}=[-, dashed]
\tikzstyle{dash_2}=[-, dashed, fill={rgb,255: red,246; green,235; blue,255}]
\tikzstyle{dash_3}=[-, dashed, fill={rgb,255: red,229; green,255; blue,181}]
\tikzstyle{dash_4}=[-, dashed, fill={rgb,255: red,255; green,209; blue,153}]
\tikzstyle{red_thick}=[-, ultra thick, draw=red]
\tikzstyle{dash_5}=[-, dashed, fill={rgb,255: red,225; green,255; blue,254}]
\tikzstyle{jellyfish}=[-, ultra thick, fill={rgb,255: red,34; green,48; blue,255}]

\usepackage{bm}
\usepackage{nicematrix}
\usepackage{comment}
\usepackage{tabularray}

\DeclareMathOperator{\Hom}{Hom}

\DeclareMathOperator{\Bell}{B}

\DeclareMathOperator{\sgn}{sgn}
\DeclareMathOperator{\Rep}{Rep}

\theoremstyle{plain}
\newtheorem{theorem}{Theorem}[section]

\theoremstyle{definition}

\theoremstyle{remark}
\newtheorem{remark}[theorem]{Remark}
\newtheorem{example}[theorem]{Example}

\newtheorem{defn}[theorem]{Definition}

\title{Categorification of Group Equivariant Neural Networks}

\author{%
  Edward Pearce--Crump\\
  Department of Computing\\
  Imperial College London\\
  London, SW7 2AZ, United Kingdom\\
  \texttt{ep1011@ic.ac.uk} \\
}

\begin{document}

\maketitle

\begin{abstract}
	We present a novel application of category theory for deep learning.
	We show how category theory can be used
	to understand and work with the linear layer functions of group equivariant neural networks whose layers are some tensor power space of $\mathbb{R}^{n}$ for the groups $S_n$, $O(n)$, $Sp(n)$, and $SO(n)$. 
	By using category theoretic constructions, we build a richer structure that is not seen in the original formulation of these neural networks, leading to new insights.	
	In particular, we outline the development of an algorithm for quickly computing the result of a vector 
	that is
	passed through an equivariant, linear layer for each group in question. 
	The success of our approach suggests that category theory could be beneficial for other areas of deep learning.
\end{abstract}

\section{Introduction}

Despite the numerous advances that have been made in many areas of deep learning, it is well known that the field is lacking a rigourous theoretical framework to support the applications that are being developed.
Practitioners typically spend a significant amount of time and effort searching for a neural network architecture that works well for the problem that they wish to solve.
The architectures are often designed using heuristics that have been shown to work well in practice, despite them being poorly understood in theory.
Notably, small modifications to the architecture can often result in a significant reduction in performance.

It has only become apparent very recently to researchers in the deep learning community that there is potential for category theory to provide a new set of tools for developing the theory of deep learning.
The hope is that category theory will provide the rigourous theoretical framework in which all existing and future results can be placed.
Category theory is based on the core concept of \textit{compositionality}; that complex systems can be built out of smaller parts, and that the entire system can be understood by studying these smaller parts.
Category theory was first used in pure mathematics in the 1940s
as a way of establishing a higher-level structure for understanding a number of algebraic objects (sets, vector spaces, topological spaces etc.) and their maps (functions, linear maps, continuous maps etc.) that shared similar characteristics.
It has since been applied successfully to many other disciplines of science, such as physics, chemistry and computer science.
Given that many deep learning architectures share similar characteristics, 
in that they are typically built out of layers and maps between these layers, it is no surprise that deep learning researchers are looking to 
category theory to achieve a similar outcome for their own field.


In this paper, we present a novel application of category theory for deep learning. 
We show that a number of group equivariant neural networks 
-- for the groups $S_n$, $O(n)$, $Sp(n)$ and $SO(n)$ --  
whose layers are some tensor power of $\mathbb{R}^{n}$
that have recently appeared in the literature 
(\cite{maron2018, pearcecrump, pearcecrumpB, godfrey})
can be understood in category theoretic terms.
We call this the "Categorification of Group Equivariant Neural Networks" because, in proving this result, we replace a number of set-theoretic constructions with category theoretic notions
that results in a deeper structure for understanding and working with the layer functions of the neural networks themselves. 
We wish to emphasise that the outcome of this process is not simply a case of 
rewriting the existing results in a different language,
but that, crucially, we 
obtain new insights into these neural networks 
from the richer structure that is established.
One particularly important consequence that we show is that any of the weight matrices that appear in the neural networks in question can be understood solely by using a certain type of combinatorial diagram that has a string-like quality to it.
By pulling on the strings or dragging their ends to different locations, we can use category theory to obtain new results for these group equivariant neural networks.
We describe a very powerful example of this idea, where the properties of the categorification 
lead to a recovery of the algorithm 
-- using a very different method --
proposed by \cite{godfrey}
for computing the result of a vector that is passed through a symmetric group equivariant linear layer. 
We suggest that our approach can be adapted to obtain an algorithm for computing the same procedure for the other groups mentioned in this paper; this result will appear in another paper by the same authors.

\section{Preliminaries}

\begin{figure}[t]
	\begin{center}
	\scalebox{0.4}{\begin{tikzpicture}
	\begin{pgfonlayer}{nodelayer}
		\node [style=none] (94) at (3, 1) {$1$};
		\node [style=none] (95) at (6, 1) {$2$};
		\node [style=none] (96) at (9, 1) {$3$};
		\node [style=none] (97) at (12, 1) {$4$};
		\node [style=none] (98) at (0, -5) {$5$};
		\node [style=none] (99) at (3, -5) {$6$};
		\node [style=none] (100) at (6, -5) {$7$};
		\node [style=none] (101) at (9, -5) {$8$};
		\node [style=none] (102) at (12, -5) {$9$};
		\node [style=none] (103) at (15, -5) {$10$};
		\node [style=none] (104) at (-0.5, 0.5) {};
		\node [style=none] (105) at (-0.5, -4.5) {};
		\node [style=none] (106) at (15.5, -4.5) {};
		\node [style=none] (107) at (15.5, 0.5) {};
		\node [style=node] (108) at (3, 0) {};
		\node [style=node] (109) at (6, 0) {};
		\node [style=node] (110) at (9, 0) {};
		\node [style=node] (111) at (12, 0) {};
		\node [style=node] (112) at (0, -4) {};
		\node [style=node] (113) at (3, -4) {};
		\node [style=node] (114) at (6, -4) {};
		\node [style=node] (115) at (9, -4) {};
		\node [style=node] (116) at (12, -4) {};
		\node [style=node] (117) at (15, -4) {};
		\node [style=none] (118) at (24, 1) {$1$};
		\node [style=none] (119) at (27, 1) {$2$};
		\node [style=none] (120) at (30, 1) {$3$};
		\node [style=none] (121) at (33, 1) {$4$};
		\node [style=none] (122) at (21, -5) {$5$};
		\node [style=none] (123) at (24, -5) {$6$};
		\node [style=none] (124) at (27, -5) {$7$};
		\node [style=none] (125) at (30, -5) {$8$};
		\node [style=none] (126) at (33, -5) {$9$};
		\node [style=none] (127) at (36, -5) {$10$};
		\node [style=node] (132) at (24, 0) {};
		\node [style=node] (133) at (27, 0) {};
		\node [style=node] (134) at (30, 0) {};
		\node [style=node] (135) at (33, 0) {};
		\node [style=node] (136) at (21, -4) {};
		\node [style=node] (137) at (24, -4) {};
		\node [style=node] (138) at (27, -4) {};
		\node [style=node] (139) at (30, -4) {};
		\node [style=node] (140) at (33, -4) {};
		\node [style=node] (141) at (36, -4) {};
		\node [style=none] (142) at (20.5, 0.5) {};
		\node [style=none] (143) at (20.5, -4.5) {};
		\node [style=none] (144) at (36.5, -4.5) {};
		\node [style=none] (145) at (36.5, 0.5) {};
		\node [style=none] (146) at (-3, -2) {\huge{b)}};
		\node [style=none] (147) at (18, -2) {\huge{c)}};
		\node [style=none] (148) at (-18, 1) {$1$};
		\node [style=none] (149) at (-15, 1) {$2$};
		\node [style=none] (150) at (-12, 1) {$3$};
		\node [style=none] (151) at (-9, 1) {$4$};
		\node [style=none] (152) at (-21, -5) {$5$};
		\node [style=none] (153) at (-18, -5) {$6$};
		\node [style=none] (154) at (-15, -5) {$7$};
		\node [style=none] (155) at (-12, -5) {$8$};
		\node [style=none] (156) at (-9, -5) {$9$};
		\node [style=none] (157) at (-6, -5) {$10$};
		\node [style=none] (158) at (-21.5, 0.5) {};
		\node [style=none] (159) at (-21.5, -4.5) {};
		\node [style=none] (160) at (-5.5, -4.5) {};
		\node [style=none] (161) at (-5.5, 0.5) {};
		\node [style=node] (162) at (-18, 0) {};
		\node [style=node] (163) at (-15, 0) {};
		\node [style=node] (164) at (-12, 0) {};
		\node [style=node] (165) at (-9, 0) {};
		\node [style=node] (166) at (-21, -4) {};
		\node [style=node] (167) at (-18, -4) {};
		\node [style=node] (168) at (-15, -4) {};
		\node [style=node] (169) at (-12, -4) {};
		\node [style=node] (170) at (-9, -4) {};
		\node [style=node] (171) at (-6, -4) {};
		\node [style=none] (172) at (-24, -2) {\huge{a)}};
	\end{pgfonlayer}
	\begin{pgfonlayer}{edgelayer}
		\draw [style={dash_2}] (105.center)
			 to (106.center)
			 to (107.center)
			 to (104.center)
			 to cycle;
		\draw [style=thick] (112) to (108);
		\draw [style=thick, bend left=60, looseness=1.25] (113) to (114);
		\draw [style=thick] (115) to (109);
		\draw [style=thick, bend right=60, looseness=1.25] (110) to (111);
		\draw [style=thick, bend left=60, looseness=1.25] (116) to (117);
		\draw [style={dash_3}] (143.center)
			 to (144.center)
			 to (145.center)
			 to (142.center)
			 to cycle;
		\draw [style=thick] (137) to (133);
		\draw [style=thick, bend left=60, looseness=1.25] (139) to (140);
		\draw [style={dash_4}] (161.center)
			 to (158.center)
			 to (159.center)
			 to (160.center)
			 to cycle;
		\draw [style=thick, bend left=45, looseness=0.75] (166) to (168);
		\draw [style=thick] (168) to (164);
		\draw [style=thick] (162) to (169);
		\draw [style=thick] (167) to (163);
		\draw [style=thick, bend left=60, looseness=1.25] (170) to (171);
		\draw [style=thick] (170) to (165);
	\end{pgfonlayer}
\end{tikzpicture}}
		\caption{Examples of $(6,4)$--partition diagrams. b) is also a $(6,4)$--Brauer diagram, and c) is also a $10 \backslash 6$--diagram.}
	\label{partdiagrams1}
	\end{center}
\end{figure}
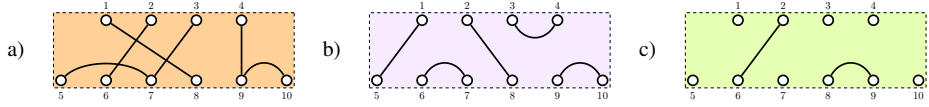

We choose our field of scalars to be $\mathbb{R}$ throughout. 
Tensor products are also taken over $\mathbb{R}$, unless otherwise stated.
Also, we let $[n]$ represent the set $\{1, \dots, n\}$. 

Recall that a representation of a group $G$ is a choice of vector space $V$ over $\mathbb{R}$ and a group homomorphism
\begin{equation} \label{grouprephom}
	\rho_V : G \rightarrow GL(V)	
\end{equation}
Furthermore, recall that a map $\phi : V \rightarrow W$ between two representations of $G$ is said to be $G$-equivariant if, for all $g \in G$ and $v \in V$, 
\begin{equation} \label{Gequivmapdefn}
	\phi(\rho_{V}(g)[v]) = \rho_{W}(g)[\phi(v)]
\end{equation}
We denote the set of all \textit{linear} $G$-equivariant maps between $V$ and $W$ by $\Hom_{G}(V,W)$. 
It can be shown that $\Hom_{G}(V,W)$ is a vector space over $\mathbb{R}$. 
See \cite{segal} for more details. 

\subsection{Tensor Power Spaces as Group Representations} \label{tenspowerspaces}

The groups of interest, namely, $S_n$, $O(n)$, $Sp(n)$, and $SO(n)$,
can all be viewed as subgroups of $GL(n)$. We use the symbol $G$ to refer to any of these groups in the following.
Recall that $\mathbb{R}^{n}$ has a standard basis that is given by $\{e_i \mid i \in [n]\}$, where $e_i$ has a $1$ in the $i^{\text{th}}$ position and is $0$ otherwise.

(Note that if $G = Sp(n)$, then $n = 2m$, and we label the indices by
$1, 1', \dots, m, m'$
and call the standard basis of $\mathbb{R}^{n}$ the symplectic basis.)

There exists a (left) action of $G$ on $\mathbb{R}^{n}$ that is given by left multiplication on the standard basis, which can be extended linearly to obtain a representation $G \rightarrow GL(\mathbb{R}^n)$.

Moreover, 
since the elements 
\begin{equation} \label{tensorelementfirst}
	e_I \coloneqq e_{i_1} \otimes e_{i_2} \otimes \dots \otimes e_{i_k} 
\end{equation}
for all $I \coloneqq (i_1, i_2, \dots, i_k) \in [n]^k$ form a basis of 
$(\mathbb{R}^{n})^{\otimes k}$,
the $k$--tensor power space of $\mathbb{R}^{n}$,
there also exists a (left) action of $G$ on $(\mathbb{R}^{n})^{\otimes k}$ that is given by
\begin{equation} 
	g \cdot e_I \coloneqq ge_{i_1} \otimes ge_{i_2} \otimes \dots \otimes ge_{i_k} 
\end{equation}
Again, this action can be extended linearly to obtain a representation $\rho_k: G \rightarrow GL((\mathbb{R}^n)^{\otimes k})$.

We are interested in the space of $G$--equivariant linear maps
between any two tensor power spaces of $\mathbb{R}^{n}$, 
$\Hom_G((\mathbb{R}^{n})^{\otimes k}, (\mathbb{R}^{n})^{\otimes l})$,
since these maps are the linear layer functions in the group equivariant neural networks of interest.


\subsection{Set Partition Diagrams} \label{setpartdiag}

\cite{pearcecrump, pearcecrumpB} showed that, for the groups $G$ in question, 
$\Hom_G((\mathbb{R}^{n})^{\otimes k}, (\mathbb{R}^{n})^{\otimes l})$
can be constructed from certain set partitions of $[l+k]$, and in particular, from their corresponding set partition diagrams. We review these constructions below.

For $l, k \in \mathbb{Z}_{\geq 0}$, consider the set 
$[l+k] \coloneqq \{1, \dots, l+k\}$
having $l+k$ elements. 
We can create a set partition of $[l+k]$ by partitioning it into a number of subsets.
We call the subsets of a set partition \textit{blocks}.
Let $\Pi_{l+k}$ be the set of all set partitions of $[l+k]$.
Then, for each set partition $\pi$ in $\Pi_{l+k}$, we can associate to it a diagram $d_\pi$,
called a $(k,l)$--partition diagram,
consisting of two
rows of vertices and edges between vertices such that there are
\begin{itemize}
	\item $l$ vertices on the top row, labelled left to right by $1, \dots, l$
	\item $k$ vertices on the bottom row, labelled left to right by $l+1, \dots, l+k$, and
	\item the edges between the vertices correspond to the connected components of $\pi$. 
\end{itemize}
As a result, $d_\pi$ represents the equivalence class of all diagrams with connected components equal to the blocks of $\pi$.

There are special types of $(k,l)$--partition diagrams that we are interested in, namely:
\begin{itemize}
	\item A $(k,l)$--Brauer diagram $d_\beta$ is a $(k,l)$--partition diagram where the size of every block in $\beta$ is exactly two.
	\item Given $k$ and $l$, an $(l+k)\backslash n$--diagram $d_\alpha$ is a $(k,l)$--partition diagram where exactly $n$ blocks in $\alpha$ have size one, with the rest having exactly size two. The vertices corresponding to the blocks of size one are called free vertices.
\end{itemize}

We give examples of these diagrams in Figure \ref{partdiagrams1}.

We can form a number of vector spaces as the $\mathbb{R}$-linear span of certain subsets of $(k,l)$--partition diagrams, as follows:
\begin{itemize}
	\item The partition space $P_k^l(n)$ is defined to be the $\mathbb{R}$-linear span of the set of all $(k,l)$--partition diagrams.
	\item The Brauer space $B_k^l(n)$ is defined to be the $\mathbb{R}$-linear span of the set of all $(k,l)$--Brauer diagrams.
	\item The Brauer--Grood space $D_k^l(n)$ is defined to be the $\mathbb{R}$-linear span of the set of all $(k,l)$--Brauer diagrams together with the set of all $(l+k)\backslash n$--diagrams.
\end{itemize}


Furthermore, we can define two $\mathbb{R}$-bilinear operations on $(k,l)$--partition diagrams
\begin{align}
	\text{composition:   } & \bullet: P_l^m(n) \times P_k^l(n) \rightarrow P_k^m(n) \label{composition} \\
	\text{tensor product:   } & \otimes: P_k^l(n) \times P_q^m(n) \rightarrow P_{k+q}^{l+m}(n) \label{tensorprod}
\end{align}
as follows:
		
		Composition: Let $d_{\pi_1} \in P_k^l(n)$ and $d_{\pi_2} \in P_l^m(n)$. 
		First, we concatenate the diagrams, written $d_{\pi_2} \circ d_{\pi_1}$, by putting $d_{\pi_1}$ below $d_{\pi_2}$, concatenating the edges in the middle row of vertices, and then removing all connected components that lie entirely in the middle row of the concatenated diagrams. 
Let $c(d_{\pi_2}, d_{\pi_1})$ be the number of connected components that are removed from the middle row in $d_{\pi_2} \circ d_{\pi_1}$.
		Then the composition is defined, using infix notation, as
\begin{equation} 
	d_{\pi_2} \bullet d_{\pi_1} 
	\coloneqq 
	n^{c(d_{\pi_2}, d_{\pi_1})} (d_{\pi_2} \circ d_{\pi_1})
\end{equation}
		Tensor Product: Let $d_{\pi_1} \in P_k^l(n)$ and $d_{\pi_2} \in P_q^m(n)$. Then $d_{\pi_1} \otimes d_{\pi_2}$ is defined to be the $(k+q, l+m)$--partition diagram obtained by horizontally placing $d_{\pi_1}$ to the left of $d_{\pi_2}$ without any overlapping of vertices.

It is clear that both of these operations are associative.

The composition and tensor product operations for $B_k^l(n)$ 
are inherited from the composition and tensor product operations for $P_k^l(n)$, defined in 
(\ref{composition}) and (\ref{tensorprod}) respectively.
However, the composition and tensor product operations for $D_k^l(n)$ are rather more involved; full details of their formulation can be found in the Technical Appendix.

\subsection{Group Equivariant Linear Layers} 
\label{groupequivlinlayers}

From the vector spaces defined in Section \ref{setpartdiag}, 
it is possible to obtain either a spanning set or a basis for 
$\Hom_G((\mathbb{R}^{n})^{\otimes k}, (\mathbb{R}^{n})^{\otimes l})$.
We give the form of the 
the spanning sets/bases, expressed in the basis of matrix units for
$\Hom((\mathbb{R}^{n})^{\otimes k}, (\mathbb{R}^{n})^{\otimes l})$,
in the Technical Appendix. 
Here, we reproduce a number of results which describe the existence of a surjective map from each of the vector spaces defined in Section \ref{setpartdiag} onto its corresponding vector space of $G$--equivariant linear maps, which arises from the spanning sets/bases.

\begin{theorem}[Diagram Basis when $G = S_n$] 
	\cite[Theorem 5.4]{godfrey}
	\label{diagbasisSn}
	
	For any $k, l \in \mathbb{Z}_{\geq 0}$ and any
	$n \in \mathbb{Z}_{\geq 1}$, there is a surjection of vector spaces
	\begin{equation} \label{surjectionSn}
		\Theta_{k,n}^l : P_k^l(n) \rightarrow 
		\Hom_{S_n}((\mathbb{R}^{n})^{\otimes k}, (\mathbb{R}^{n})^{\otimes l})
	\end{equation}
	that is given by
	\begin{equation}
		d_\pi \mapsto E_\pi
	\end{equation}
	for all $(k,l)$--partition diagrams $d_\pi$, where $E_\pi$ is given in the Technical Appendix.
	
	In particular, the set
	\begin{equation} \label{klSnSpanningSet}
		\{E_\pi \mid d_\pi \text{ is a } (k,l) \text{--partition diagram having at most } n \text{ blocks} \}
	\end{equation}
	is a basis for
	$\Hom_{S_n}((\mathbb{R}^{n})^{\otimes k}, (\mathbb{R}^{n})^{\otimes l})$
	in the standard basis of $\mathbb{R}^{n}$, 
	of size $\Bell(l+k,n) \coloneqq
	\sum_{t=1}^{n} 
		\begin{Bsmallmatrix}
		l+k\\
		t 
		\end{Bsmallmatrix}
	$, where 
	$
		\begin{Bsmallmatrix}
		l+k\\
		t 
		\end{Bsmallmatrix}
	$
	is the Stirling number of the second kind.
\end{theorem}

\begin{theorem}
	[Spanning set when $G = O(n)$]
	\cite[Theorem 6.5]{pearcecrumpB}
	\label{spanningsetO(n)}

	For any $k, l \in \mathbb{Z}_{\geq 0}$ and any
	$n \in \mathbb{Z}_{\geq 1}$, there is a surjection of vector spaces
	\begin{equation} \label{surjectionO(n)}
		\Phi_{k,n}^l : B_k^l(n) \rightarrow 
		\Hom_{O(n)}((\mathbb{R}^{n})^{\otimes k}, (\mathbb{R}^{n})^{\otimes l})
	\end{equation}
	that is given by
	\begin{equation}
		d_\beta \mapsto E_\beta
	\end{equation}
	for all $(k,l)$--Brauer diagrams $d_\beta$, where $E_\beta$ 
	is given in the Technical Appendix.
	
	In particular, the set
	\begin{equation} \label{klOnSpanningSet}
		\{E_\beta \mid d_\beta \text{ is a } (k,l) \text{--Brauer diagram} \}
	\end{equation}
	is a spanning set for
	$\Hom_{O(n)}((\mathbb{R}^{n})^{\otimes k}, (\mathbb{R}^{n})^{\otimes l})$
	in the standard basis of $\mathbb{R}^{n}$, 
	of size $0$ when $l+k$ is odd, and of size $(l+k-1)!!$ when $l+k$ is even.
\end{theorem}

\begin{theorem} 
	[Spanning set when $G = Sp(n), n = 2m$]
	\cite[Theorem 6.6]{pearcecrumpB}
	\label{spanningsetSp(n)}

	For any $k, l \in \mathbb{Z}_{\geq 0}$ and any
	$n \in \mathbb{Z}_{\geq 2}$ such that $n = 2m$, there is a surjection of vector spaces
	\begin{equation} \label{surjectionSp(n)}
		X_{k,n}^l : B_k^l(n) \rightarrow 
		\Hom_{Sp(n)}((\mathbb{R}^{n})^{\otimes k}, (\mathbb{R}^{n})^{\otimes l})
	\end{equation}
	that is given by
	\begin{equation}
		d_\beta \mapsto F_\beta
	\end{equation}
	for all $(k,l)$--Brauer diagrams $d_\beta$, where $F_\beta$ is 
	given in the Technical Appendix.

	In particular, the set
	\begin{equation} \label{klSpnSpanningSet}
		\{F_\beta \mid d_\beta \text{ is a } (k,l) \text{--Brauer diagram} \}
	\end{equation}
	is a spanning set for
	$\Hom_{Sp(n)}((\mathbb{R}^{n})^{\otimes k}, (\mathbb{R}^{n})^{\otimes l})$, for $n = 2m$,
	in the symplectic basis of $\mathbb{R}^{n}$,
	of size $0$ when $l+k$ is odd, and of size $(l+k-1)!!$ when $l+k$ is even.
\end{theorem}

\begin{theorem} 
	[Spanning set when $G = SO(n)$]
	\cite[Theorem 6.7]{pearcecrumpB}
	\label{spanningsetSO(n)}

	For any $k, l \in \mathbb{Z}_{\geq 0}$ and any $n \in \mathbb{Z}_{\geq 1}$, there is a surjection of vector spaces
	\begin{equation} \label{surjectionSO(n)}
		\Psi_{k,n}^l : D_k^l(n) \rightarrow 
		\Hom_{SO(n)}((\mathbb{R}^{n})^{\otimes k}, (\mathbb{R}^{n})^{\otimes l})
	\end{equation}
	that is given by
	\begin{equation}
		d_\beta \mapsto E_\beta
	\end{equation}
	if $d_\beta$ is a $(k,l)$--Brauer diagram, where $E_\beta$
	is given in the Technical Appendix,
	and by
	\begin{equation} \label{surjdalpha}
		d_\alpha \mapsto H_\alpha
	\end{equation}
	if $d_\alpha$ is a $(k+l)\backslash n$--diagram, where $H_\alpha$ is 
	is also given in the Technical Appendix
	
	In particular, the set
	\begin{equation} \label{SOn2SpanningSet}
		\{E_\beta\}_{\beta}
		\cup
		\{H_\alpha\}_{\alpha}
	\end{equation}
	where $d_\beta$ is a $(k,l)$--Brauer diagram, and $d_\alpha$ is a $(l+k) \backslash n$--diagram,
	is a spanning set for
	$\Hom_{SO(n)}((\mathbb{R}^{n})^{\otimes k}, (\mathbb{R}^{n})^{\otimes l})$
	in the standard basis of $\mathbb{R}^{n}$.
\end{theorem}

\section{Strict $\mathbb{R}$--Linear Monoidal Categories and String Diagrams}

We appreciate that the language of category theory is not commonplace in the machine learning literature. To aid the reader, we have provided some foundational material in the Technical Appendix. 
Other good references are \cite{maclane, kock, turaev}.

At a basic level, category theory is concerned with objects and the relationships between objects. These relationships are called morphisms. 
A collection of objects and a collection of morphisms between objects (satisfying some additional conditions) form a category.
We can perform operations with morphisms, such as (vertically) composing them, to form new morphisms between objects.
Category theory makes it possible to abstract away specific details of structures, to focus instead on the relationships between them. 
We are interested not only in the relationships within a category but also in how relationships are preserved across different categories. These are described by functors.

In this paper, we are interested in categories that have a specific property called \textit{monoidal}, and the (monoidal) functors between these categories.
We will see in Section \ref{implicationsgroupequiv} that it is this property that has important implications for the group equivariant neural networks that we look at in this paper.
The monoidal property gives additional structure to the way in which objects and morphisms can be related.
In particular, monoidal categories have an additional operation, known as a tensor product, that allows objects and morphisms to be composed in a different way, which we call horizontal. 
Monoidal functors preserve the tensor product across monoidal categories.


We assume throughout that all categories are \textit{locally small}; that is, that the collection of morphisms between any two objects is a set.
In fact, all of the categories that we consider in this paper have morphism sets that are vector spaces: in particular, the morphisms between objects become linear maps.
We follow the presentation given in 
\cite{Hu2019} and \cite{Savage2021}
below.

\subsection{Strict $\mathbb{R}$--Linear Monoidal Categories}

\begin{defn} \label{categorystrictmonoidal}
	A category $\mathcal{C}$ is said to be \textit{strict monoidal} if it comes with
		a bifunctor $\otimes: \mathcal{C} \times \mathcal{C} \rightarrow \mathcal{C}$, called the tensor product, and
	a unit object $\mathds{1}$,
	such that, for all objects $X, Y, Z$ in $\mathcal{C}$, we have that
	\begin{equation}
		(X \otimes Y) \otimes Z = X \otimes (Y \otimes Z)
	\end{equation}
	\begin{equation}
		(\mathds{1} \otimes X) = X = (X \otimes \mathds{1})
	\end{equation}
	and, for all morphisms $f, g, h$ in $\mathcal{C}$, we have that
	\begin{equation} \label{assocbifunctor}
		(f \otimes g) \otimes h = f \otimes (g \otimes h)
	\end{equation}
	\begin{equation}
		(1_\mathds{1} \otimes f) = f = (f \otimes 1_\mathds{1})
	\end{equation}
	where $1_\mathds{1}$ is the identity morphism $\mathds{1} \rightarrow \mathds{1}$.
\end{defn}

\begin{remark}
	We can assume that all monoidal categories are strict (nonstrict monoidal categories would have isomorphisms in the place of the equalities given in Definition \ref{categorystrictmonoidal}) owing to a technical result known as Mac Lane's Coherence Theorem. See \cite{maclane} for more details.
\end{remark}

\begin{defn} \label{categorylinear}
	A category $\mathcal{C}$ is said to be $\mathbb{R}$\textit{--linear} if,
		for any two objects $X, Y$ in $\mathcal{C}$, the morphism space $\Hom_{\mathcal{C}}(X,Y)$ is a vector space over $\mathbb{R}$, and
		the composition of morphisms is $\mathbb{R}$--bilinear.
\end{defn}

Combining Definitions \ref{categorystrictmonoidal}
and \ref{categorylinear},
we get
\begin{defn}
	A category $\mathcal{C}$ is said to be \textit{strict} $\mathbb{R}$\textit{--linear monoidal} if it is a category that is both strict monoidal and $\mathbb{R}$--linear, such that the bifunctor $\otimes$ is $\mathbb{R}$--bilinear.
\end{defn}

Analogous to how there exists maps between sets, there exists "maps" between categories, known as functors.
In particular, we are interested in the following type of functors:
\begin{defn} \label{monoidalfunctordefn}
	Suppose that
	$(\mathcal{C}, \otimes_\mathcal{C}, \mathds{1}_\mathcal{C})$
	and
	$(\mathcal{D}, \otimes_\mathcal{D}, \mathds{1}_\mathcal{D})$
	are two strict $\mathbb{R}$--linear monoidal categories.

	A \textit{strict} $\mathbb{R}$\textit{--linear monoidal functor} from $\mathcal{C}$ to $\mathcal{D}$ is a functor $\mathcal{F}: \mathcal{C} \rightarrow \mathcal{D}$ such that
	\begin{enumerate}
		\item for all objects $X, Y$ in $\mathcal{C}$,
			$\mathcal{F}(X \otimes_\mathcal{C} Y) =
			\mathcal{F}(X) \otimes_\mathcal{D} \mathcal{F}(Y)$
		\item for all morphisms $f, g$ in $\mathcal{C}$,
			$\mathcal{F}(f \otimes_\mathcal{C} g) =
			\mathcal{F}(f) \otimes_\mathcal{D} \mathcal{F}(g)$
		\item $\mathcal{F}(\mathds{1}_\mathcal{C}) = \mathds{1}_\mathcal{D}$, and
		\item for all objects $X, Y$ in $\mathcal{C}$, the map
		\begin{equation} \label{maphomsets}
			\Hom_{\mathcal{C}}(X,Y) 
			\rightarrow 
			\Hom_{\mathcal{D}}(\mathcal{F}(X),\mathcal{F}(Y))
		\end{equation}
		given by
		$f \mapsto \mathcal{F}(f)$
		is $\mathbb{R}$--linear.
	\end{enumerate}
\end{defn}

The following definition will also prove to be very important in what follows.
\begin{defn}
	Given any two (locally small) categories $\mathcal{C}$ and $\mathcal{D}$
	(not necessarily strict $\mathbb{R}$--linear monoidal), 
	a functor $\mathcal{F}: \mathcal{C} \rightarrow \mathcal{D}$
	is said to be \textit{full} if the map (\ref{maphomsets}) is surjective for all objects $X, Y$ in $\mathcal{C}$.
\end{defn}

\subsection{String Diagrams}

Strict monoidal categories are particularly interesting because they can be represented by a very useful diagrammatic language known as string diagrams.
As this language is, in some sense, geometric in nature,
we will see that it is much easier to work with these diagrams than with their equivalent 
algebraic form.

\begin{defn} [String Diagrams]
	Suppose that $\mathcal{C}$ is a strict monoidal category.
	Let $W, X, Y$ and $Z$ be objects in $\mathcal{C}$, and let
$f: X \rightarrow Y$, $g: Y \rightarrow Z$, and $h: W \rightarrow Z$ be morphisms in $\mathcal{C}$.
	Then we can represent the morphisms
	$1_{X}: X \rightarrow X$,
	$f: X \rightarrow Y$,
	$g \circ f: X \rightarrow Z$ and
	$f \otimes h: X \otimes W \rightarrow Y \otimes Z$
	as diagrams in the following way:
	\begin{equation}
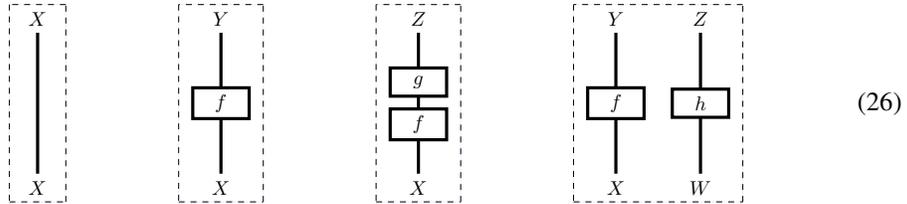
		\begin{aligned}
		\scalebox{0.75}{\begin{tikzpicture}
	\begin{pgfonlayer}{nodelayer}
		\node [style=none] (1) at (0, 3) {};
		\node [style=none] (2) at (0, -4) {};
		\node [style=none] (3) at (2, -4) {};
		\node [style=none] (4) at (2, 3) {};
		\node [style=none] (36) at (6, 3) {};
		\node [style=none] (37) at (6, -4) {};
		\node [style=none] (38) at (9, -4) {};
		\node [style=none] (39) at (9, 3) {};
		\node [style=large box] (42) at (7.5, -0.5) {$f$};
		\node [style=none] (43) at (13, 3) {};
		\node [style=none] (44) at (13, -4) {};
		\node [style=none] (45) at (16, -4) {};
		\node [style=none] (46) at (16, 3) {};
		\node [style=large box] (49) at (14.5, -1.25) {$f$};
		\node [style=large box] (50) at (14.5, 0.25) {$g$};
		\node [style=none] (51) at (20, 3) {};
		\node [style=none] (52) at (20, -4) {};
		\node [style=none] (53) at (26, -4) {};
		\node [style=none] (54) at (26, 3) {};
		\node [style=large box] (57) at (21.5, -0.5) {$f$};
		\node [style=large box] (58) at (24.5, -0.5) {$h$};
		\node [style=none] (62) at (1, 2) {};
		\node [style=none] (63) at (1, -3) {};
		\node [style=none] (64) at (1, -3.5) {$X$};
		\node [style=none] (65) at (1, 2.5) {$X$};
		\node [style=none] (66) at (7.5, 2.5) {$Y$};
		\node [style=none] (67) at (7.5, -3.5) {$X$};
		\node [style=none] (68) at (7.5, 2) {};
		\node [style=none] (69) at (7.5, -3) {};
		\node [style=none] (70) at (14.5, 2.5) {$Z$};
		\node [style=none] (71) at (14.5, -3.5) {$X$};
		\node [style=none] (72) at (14.5, 2) {};
		\node [style=none] (73) at (14.5, -3) {};
		\node [style=none] (74) at (24.5, 2.5) {$Z$};
		\node [style=none] (75) at (21.5, -3.5) {$X$};
		\node [style=none] (76) at (24.5, -3.5) {$W$};
		\node [style=none] (77) at (21.5, 2.5) {$Y$};
		\node [style=none] (78) at (21.5, 2) {};
		\node [style=none] (79) at (21.5, -3) {};
		\node [style=none] (80) at (24.5, 2) {};
		\node [style=none] (81) at (24.5, -3) {};
	\end{pgfonlayer}
	\begin{pgfonlayer}{edgelayer}
		\draw [style={dash_2}] (2.center) to (3.center);
		\draw [style={dash_2}] (3.center) to (4.center);
		\draw [style={dash_2}] (4.center) to (1.center);
		\draw [style={dash_2}] (1.center) to (2.center);
		\draw [style={dash_2}] (37.center) to (38.center);
		\draw [style={dash_2}] (38.center) to (39.center);
		\draw [style={dash_2}] (39.center) to (36.center);
		\draw [style={dash_2}] (36.center) to (37.center);
		\draw [style={dash_2}] (46.center) to (43.center);
		\draw [style={dash_2}] (43.center) to (44.center);
		\draw [style={dash_2}] (44.center) to (45.center);
		\draw [style={dash_2}] (45.center) to (46.center);
		\draw [style={dash_2}] (54.center) to (51.center);
		\draw [style={dash_2}] (51.center) to (52.center);
		\draw [style={dash_2}] (52.center) to (53.center);
		\draw [style={dash_2}] (53.center) to (54.center);
		\draw [style=thick] (62.center) to (63.center);
		\draw [style=thick] (69.center) to (68.center);
		\draw [style=thick] (73.center) to (72.center);
		\draw [style=thick] (79.center) to (78.center);
		\draw [style=thick] (81.center) to (80.center);
	\end{pgfonlayer}
\end{tikzpicture}}
		\end{aligned}
	\end{equation}
	In particular, the vertical composition of morphisms $g \circ f$ is obtained by placing $g$ above $f$, and the horizontal composition of morphisms $f \otimes h$ is obtained by horizontally placing $f$ to the left of $h$.

	We will often omit the labelling of the objects when they are clear or when they are not important.
\end{defn}

As an example of how useful string diagrams are when working with strict monoidal categories, the associativity of the bifunctor given in (\ref{assocbifunctor}) becomes immediately apparent.
Another, more involved, example is given by the interchange law that exists for any strict monoidal category. It can be expressed algebraically as
\begin{equation}
	(\mathds{1} \otimes g) \circ (f \otimes \mathds{1})
	=
	f \otimes g
	=
	(f \otimes \mathds{1}) \circ (\mathds{1} \otimes g) 
\end{equation}
Without string diagrams, it is somewhat tedious to prove this result -- see \cite[Section 2.2]{Savage2021} --
but with them, the result is intuitively obvious, if we allow ourselves to deform the diagrams by pulling on the strings:
\begin{equation}
	\begin{aligned}
		\scalebox{0.75}{\begin{tikzpicture}
	\begin{pgfonlayer}{nodelayer}
		\node [style=none] (0) at (0, 3.5) {};
		\node [style=none] (1) at (0, -3.5) {};
		\node [style=none] (2) at (6, -3.5) {};
		\node [style=none] (3) at (6, 3.5) {};
		\node [style=large box] (6) at (1.5, -1.25) {$f$};
		\node [style=large box] (7) at (4.5, 1.25) {$g$};
		\node [style=none] (10) at (9, 3.5) {};
		\node [style=none] (11) at (9, -3.5) {};
		\node [style=none] (12) at (15, -3.5) {};
		\node [style=none] (13) at (15, 3.5) {};
		\node [style=large box] (16) at (10.5, 0) {$f$};
		\node [style=large box] (17) at (13.5, 0) {$g$};
		\node [style=none] (20) at (7.5, 0) {$=$};
		\node [style=none] (21) at (18, 3.5) {};
		\node [style=none] (22) at (18, -3.5) {};
		\node [style=none] (23) at (24, -3.5) {};
		\node [style=none] (24) at (24, 3.5) {};
		\node [style=large box] (27) at (19.5, 1.25) {$f$};
		\node [style=large box] (28) at (22.5, -1.25) {$g$};
		\node [style=none] (31) at (16.5, 0) {$=$};
		\node [style=none] (36) at (1.5, 2.5) {};
		\node [style=none] (37) at (1.5, -2.5) {};
		\node [style=none] (38) at (22.5, 2.5) {};
		\node [style=none] (39) at (22.5, -2.5) {};
		\node [style=none] (40) at (19.5, 2.5) {};
		\node [style=none] (41) at (19.5, -2.5) {};
		\node [style=none] (42) at (13.5, 2.5) {};
		\node [style=none] (43) at (13.5, -2.5) {};
		\node [style=none] (44) at (10.5, 2.5) {};
		\node [style=none] (45) at (10.5, -2.5) {};
		\node [style=none] (46) at (4.5, 2.5) {};
		\node [style=none] (47) at (4.5, -2.5) {};
	\end{pgfonlayer}
	\begin{pgfonlayer}{edgelayer}
		\draw [style={dash_2}] (3.center) to (0.center);
		\draw [style={dash_2}] (0.center) to (1.center);
		\draw [style={dash_2}] (1.center) to (2.center);
		\draw [style={dash_2}] (2.center) to (3.center);
		\draw [style={dash_2}] (13.center) to (10.center);
		\draw [style={dash_2}] (10.center) to (11.center);
		\draw [style={dash_2}] (11.center) to (12.center);
		\draw [style={dash_2}] (12.center) to (13.center);
		\draw [style={dash_2}] (24.center) to (21.center);
		\draw [style={dash_2}] (21.center) to (22.center);
		\draw [style={dash_2}] (22.center) to (23.center);
		\draw [style={dash_2}] (23.center) to (24.center);
		\draw [style=thick] (37.center) to (36.center);
		\draw [style=thick] (39.center) to (38.center);
		\draw [style=thick] (41.center) to (40.center);
		\draw [style=thick] (43.center) to (42.center);
		\draw [style=thick] (45.center) to (44.center);
		\draw [style=thick] (47.center) to (46.center);
	\end{pgfonlayer}
\end{tikzpicture}}
	\end{aligned}
\end{equation}

\section{Categorification}

At this point, we have defined a vector space for each $k, l \in \mathbb{Z}_{\geq 0}$ that is the $\mathbb{R}$--linear span of a certain subset of $(k,l)$--partition diagrams. 
However, it should be apparent that, for all values of $k$ and $l$,
these vector spaces are all similar in nature, in that the set partition diagrams only differ by the number of vertices that appear in each row and by the connections that are made between vertices.
Moreover, the astute reader may have noticed that set partition diagrams
look like string diagrams. 
Given that string diagrams represent strict monoidal categories, and that we have
a collection of vector spaces 
for certain subsets of set partition diagrams, this implies that we should have a number of strict $\mathbb{R}$--linear monoidal categories!
Indeed we do; we formalise this intuition below.

\subsection{Category Definitions}

We assume throughout that $n \in \mathbb{Z}_{> 0}$. 
\begin{defn} \label{partitioncategory}
	We define the partition category $\mathcal{P}(n)$ to be 
	the category whose objects are the non--negative integers $\mathbb{Z}_{\geq 0} = \{0, 1, 2, \dots \}$,
	and, for any pair of objects $k$ and $l$, the morphism space 
	$\Hom_{\mathcal{P}(n)}(k,l)$ is $P_k^l(n)$.
	
	The vertical composition of morphisms is given by the composition of partition diagrams defined in (\ref{composition});
	the bifunctor (the horizontal composition of morphisms) is given by the tensor product of partition diagrams defined in (\ref{tensorprod});
	and the unit object is $0$.
\end{defn}

\begin{defn}
	We define the Brauer category $\mathcal{B}(n)$ to be the category
	whose objects are the same as those of $\mathcal{P}(n)$
	and, for any pair of objects $k$ and $l$, the morphism space 
	$\Hom_{\mathcal{B}(n)}(k,l)$ is $B_k^l(n)$.
	
	The vertical composition of morphisms, the horizontal composition of morphisms and the unit object are the same as those of $\mathcal{P}(n)$.
\end{defn}

\begin{defn}
	We define the Brauer--Grood category $\mathcal{BG}(n)$ to be the category
	whose objects are the same as those of $\mathcal{P}(n)$
	and, for any pair of objects $k$ and $l$, the morphism space 
	$\Hom_{\mathcal{BG}(n)}(k,l)$ is $D_k^l(n)$.
	
	
	The vertical composition of morphisms 
	and the horizontal composition of morphisms
	are the same as those defined for $D_k^l(n)$, which can be found in the Technical Appendix.
	The unit object is $0$.
\end{defn}
	
	It is easy to show that $\mathcal{P}(n)$, $\mathcal{B}(n)$ and $\mathcal{BG}(n)$ are strict $\mathbb{R}$--linear monoidal categories.


Also, for the four groups of interest, we can define the following category. 
\begin{defn}
	If $G$ is any of the groups $S_n, O(n), Sp(n)$ or $SO(n)$, then
	we define $\mathcal{C}(G)$ to be the category
	whose objects are pairs $\{((\mathbb{R}^n)^{\otimes k},\rho_k)\}_{k \in \mathbb{Z}_{\geq 0}}$, 
	where $\rho_k: G \rightarrow GL((\mathbb{R}^n)^{\otimes k})$ is the representation of $G$ given in Section \ref{tenspowerspaces},
	and,
	for any pair of objects
	$((\mathbb{R}^n)^{\otimes k},\rho_k)$ and $((\mathbb{R}^n)^{\otimes l},\rho_l)$,
	the morphism space, 
	$\Hom_{\mathcal{C}(G)}(((\mathbb{R}^n)^{\otimes k},\rho_k), ((\mathbb{R}^n)^{\otimes l},\rho_l))$ 
	is precisely
	$\Hom_{G}((\mathbb{R}^n)^{\otimes k}, (\mathbb{R}^n)^{\otimes l})$.

	The vertical composition of morphisms is given by the usual composition of linear maps, the 
	horizontal composition of morphisms 
	is given by the usual tensor product of linear maps, both of which are associative operations, and the unit object is given by $(\mathbb{R}, 1_\mathbb{R})$, where $1_\mathbb{R}$ is the one-dimensional trivial representation of $G$.
\end{defn}
	
	It can be shown that $\mathcal{C}(G)$ is a subcategory of the category of representations of $G$, $\Rep(G)$. See the Technical Appendix for more details. 
	In particular, it is also a strict $\mathbb{R}$--linear monoidal category.

\subsection{Full, Strict $\mathbb{R}$--Linear Monoidal Functors} \label{fullstrictmonfunct}

\begin{figure}
  \centering
	\scalebox{0.35}{\begin{tikzpicture}
	\begin{pgfonlayer}{nodelayer}
		\node [style=none] (0) at (14, -4.5) {\Huge$\bm{=}$};
		\node [style=node] (1) at (3, -9) {};
		\node [style=none] (2) at (-0.5, 0.5) {};
		\node [style=none] (3) at (-0.5, -9.5) {};
		\node [style=none] (4) at (12.5, -9.5) {};
		\node [style=none] (5) at (12.5, 0.5) {};
		\node [style=node] (6) at (3, 0) {};
		\node [style=node] (8) at (12, -9) {};
		\node [style=node] (9) at (0, -9) {};
		\node [style=node] (16) at (6, -9) {};
		\node [style=node] (17) at (9, -9) {};
		\node [style=node] (18) at (9, 0) {};
		\node [style=node] (19) at (6, 0) {};
		\node [style=node] (20) at (19, -9) {};
		\node [style=none] (21) at (15.5, 0.5) {};
		\node [style=none] (22) at (15.5, -9.5) {};
		\node [style=none] (23) at (28.5, -9.5) {};
		\node [style=none] (24) at (28.5, 0.5) {};
		\node [style=node] (25) at (19, 0) {};
		\node [style=node] (26) at (28, -9) {};
		\node [style=node] (27) at (16, -9) {};
		\node [style=node] (28) at (22, -9) {};
		\node [style=node] (29) at (25, -9) {};
		\node [style=node] (33) at (28, -6) {};
		\node [style=none] (37) at (30, -4.5) {\Huge$\bm{=}$};
		\node [style=node] (38) at (35, -9) {};
		\node [style=none] (39) at (31.5, 0.5) {};
		\node [style=none] (40) at (31.5, -9.5) {};
		\node [style=none] (41) at (44.5, -9.5) {};
		\node [style=none] (42) at (44.5, 0.5) {};
		\node [style=node] (43) at (35, 0) {};
		\node [style=node] (44) at (44, -9) {};
		\node [style=node] (45) at (32, -9) {};
		\node [style=node] (46) at (38, -9) {};
		\node [style=node] (47) at (41, -9) {};
		\node [style=node] (50) at (44, -6) {};
		\node [style=node] (51) at (41, -6) {};
		\node [style=node] (52) at (38, -6) {};
		\node [style=none] (54) at (46, -4.5) {\Huge$\bm{=}$};
		\node [style=none] (56) at (47.5, 0.5) {};
		\node [style=none] (57) at (47.5, -9.5) {};
		\node [style=none] (58) at (60.5, -9.5) {};
		\node [style=none] (59) at (60.5, 0.5) {};
		\node [style=node] (60) at (51, 0) {};
		\node [style=node] (71) at (54, -3) {};
		\node [style=node] (72) at (51, -3) {};
		\node [style=node] (73) at (41, 0) {};
		\node [style=node] (74) at (38, 0) {};
		\node [style=node] (75) at (25, 0) {};
		\node [style=node] (76) at (22, 0) {};
		\node [style=node] (77) at (54, 0) {};
		\node [style=node] (78) at (57, 0) {};
		\node [style=none] (79) at (62, -4.5) {\Huge$\bm{=}$};
		\node [style=node] (80) at (67, -9) {};
		\node [style=none] (82) at (63.5, -9.5) {};
		\node [style=none] (83) at (76.5, -9.5) {};
		\node [style=node] (85) at (67, 0) {};
		\node [style=node] (86) at (76, -9) {};
		\node [style=node] (87) at (64, -9) {};
		\node [style=node] (88) at (70, -9) {};
		\node [style=node] (89) at (73, -9) {};
		\node [style=node] (90) at (76, -6) {};
		\node [style=node] (91) at (73, -6) {};
		\node [style=node] (92) at (70, -6) {};
		\node [style=node] (93) at (67, -6) {};
		\node [style=node] (94) at (70, -3) {};
		\node [style=node] (95) at (67, -3) {};
		\node [style=node] (96) at (70, 0) {};
		\node [style=node] (97) at (73, 0) {};
		\node [style=node] (98) at (73, -3) {};
		\node [style=node] (99) at (64, -6) {};
		\node [style=none] (102) at (63.5, -6) {};
		\node [style=none] (103) at (76.5, -6) {};
		\node [style=none] (104) at (63.5, -3) {};
		\node [style=none] (105) at (63.5, -6) {};
		\node [style=none] (106) at (76.5, -6) {};
		\node [style=none] (107) at (76.5, -3) {};
		\node [style=none] (108) at (63.5, -3) {};
		\node [style=none] (109) at (76.5, -3) {};
		\node [style=none] (110) at (76.5, 0.5) {};
		\node [style=none] (111) at (63.5, 0.5) {};
		\node [style=node] (112) at (51, -9) {};
		\node [style=node] (113) at (60, -9) {};
		\node [style=node] (114) at (48, -9) {};
		\node [style=node] (115) at (54, -9) {};
		\node [style=node] (116) at (57, -9) {};
		\node [style=node] (117) at (60, -6) {};
		\node [style=node] (118) at (57, -6) {};
		\node [style=node] (119) at (54, -6) {};
	\end{pgfonlayer}
	\begin{pgfonlayer}{edgelayer}
		\draw [style={dash_2}] (22.center) to (23.center);
		\draw [style={dash_2}] (23.center) to (24.center);
		\draw [style={dash_2}] (24.center) to (21.center);
		\draw [style={dash_2}] (21.center) to (22.center);
		\draw [style=thick] (29) to (25);
		\draw [style={dash_2}] (40.center) to (41.center);
		\draw [style={dash_2}] (41.center) to (42.center);
		\draw [style={dash_2}] (42.center) to (39.center);
		\draw [style={dash_2}] (39.center) to (40.center);
		\draw [style=thick] (47) to (43);
		\draw [style={dash_2}] (59.center) to (56.center);
		\draw [style={dash_2}] (56.center) to (57.center);
		\draw [style={dash_2}] (57.center) to (58.center);
		\draw [style={dash_2}] (58.center) to (59.center);
		\draw [style=thick, bend right=60, looseness=1.25] (72) to (71);
		\draw [style=thick, bend right=60, looseness=1.25] (74) to (73);
		\draw [style=thick, bend right=60, looseness=1.25] (76) to (75);
		\draw [style=thick] (72) to (77);
		\draw [style=thick] (71) to (78);
		\draw [style={dash_5}] (110.center)
			 to (111.center)
			 to (108.center)
			 to (109.center)
			 to cycle;
		\draw [style=thick] (85) to (98);
		\draw [style=thick] (95) to (96);
		\draw [style=thick] (94) to (97);
		\draw [style={dash_5}] (103.center)
			 to (102.center)
			 to (82.center)
			 to (83.center)
			 to cycle;
		\draw [style=thick] (89) to (99);
		\draw [style={dash_4}] (4.center)
			 to (5.center)
			 to (2.center)
			 to (3.center)
			 to cycle;
		\draw [style=thick, bend left=60, looseness=1.25] (17) to (8);
		\draw [style=thick, bend right=60, looseness=1.25] (19) to (18);
		\draw [style=thick, bend left=300, looseness=1.25] (1) to (9);
		\draw [style=thick, bend left=60, looseness=1.25] (29) to (26);
		\draw [style=thick] (28) to (33);
		\draw [style=thick, bend left=60, looseness=1.25] (27) to (20);
		\draw [style=thick] (46) to (50);
		\draw [style=thick] (38) to (51);
		\draw [style=thick] (45) to (52);
		\draw [style=thick, bend left=60, looseness=1.25] (47) to (44);
		\draw [style=thick, bend left=60, looseness=1.25] (52) to (51);
		\draw [style=thick] (115) to (117);
		\draw [style=thick] (112) to (118);
		\draw [style=thick] (114) to (119);
		\draw [style=thick, bend left=60, looseness=1.25] (116) to (113);
		\draw [style=thick, bend left=60, looseness=1.25] (119) to (118);
		\draw [style=thick] (60) to (116);
		\draw [style=thick] (86) to (93);
		\draw [style=thick] (88) to (90);
		\draw [style=thick] (80) to (91);
		\draw [style=thick] (87) to (92);
		\draw [style=thick] (6) to (17);
		\draw [style={dash_4}] (107.center)
			 to (104.center)
			 to (105.center)
			 to (106.center)
			 to cycle;
		\draw [style=thick] (93) to (98);
		\draw [style=thick, bend left=60, looseness=1.25] (99) to (93);
		\draw [style=thick, bend left=60, looseness=1.25] (92) to (91);
		\draw [style=thick, bend right=60, looseness=1.25] (95) to (94);
	\end{pgfonlayer}
\end{tikzpicture}}
	\caption{We can use the string-like aspect of $(k,l)$--partition diagrams to factor them as a composition of a permutation in $S_k$, a \textit{planar} $(k,l)$--partition diagram, and a permutation in $S_l$.}
	\label{symmfactoring}
\end{figure}
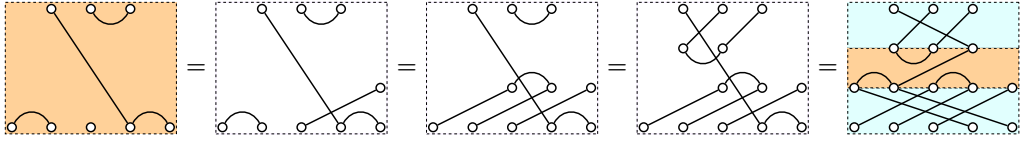

Given that we have a number of categories from the vector spaces of the different types of partition diagrams, that we have a category from the group equivariant linear maps between tensor power spaces, and that we have a number of maps between these vector spaces -- as seen in Section \ref{groupequivlinlayers} --
we should have a number of functors between the newly defined categories.
Indeed, we have that

\begin{theorem} \label{partfunctor}
	There exists a full, strict $\mathbb{R}$--linear monoidal functor
	\begin{equation}
		\Theta : \mathcal{P}(n) \rightarrow \mathcal{C}(S_n)
	\end{equation}
	that is defined on the objects of $\mathcal{P}(n)$ by 
	$\Theta(k) \coloneqq ((\mathbb{R}^{n})^{\otimes k}, \rho_k)$ 
	and, for any objects $k,l$ of $\mathcal{P}(n)$, the map
	\begin{equation} \label{partmorphism}
		\Hom_{\mathcal{P}(n)}(k,l) 
		\rightarrow 
		\Hom_{\mathcal{C}(S_n)}(\Theta(k),\Theta(l))
	\end{equation}
	is precisely the map 
	\begin{equation} \label{partmorphismmap}
		\Theta_{k,n}^l : P_k^l(n) \rightarrow 
		\Hom_{S_n}((\mathbb{R}^{n})^{\otimes k}, (\mathbb{R}^{n})^{\otimes l})
	\end{equation}
	given in Theorem \ref{diagbasisSn}.
\end{theorem}

\begin{theorem} \label{brauerO(n)functor}
	There exists a full, strict $\mathbb{R}$--linear monoidal functor
	\begin{equation}
		\Phi : \mathcal{B}(n) \rightarrow \mathcal{C}(O(n))
	\end{equation}
	that is defined on the objects of $\mathcal{B}(n)$ by 
	$\Phi(k) \coloneqq ((\mathbb{R}^{n})^{\otimes k}, \rho_k)$
	and, for any objects $k,l$ of $\mathcal{B}(n)$, the map
	\begin{equation}	
		\Hom_{\mathcal{B}(n)}(k,l) 
		\rightarrow 
		\Hom_{\mathcal{C}(O(n))}(\Phi(k),\Phi(l))
	\end{equation}
	is the map
	\begin{equation} 
		\Phi_{k,n}^l : B_k^l(n) \rightarrow 
		\Hom_{O(n)}((\mathbb{R}^{n})^{\otimes k}, (\mathbb{R}^{n})^{\otimes l})
	\end{equation}
	given in Theorem \ref{spanningsetO(n)}.
\end{theorem}

\begin{theorem} \label{brauerSp(n)functor}
	There exists a full, strict $\mathbb{R}$--linear monoidal functor
	\begin{equation}
		X : \mathcal{B}(n) \rightarrow \mathcal{C}(Sp(n))
	\end{equation}
	that is defined on the objects of $\mathcal{B}(n)$ by 
	$X(k) \coloneqq ((\mathbb{R}^{n})^{\otimes k}, \rho_k)$
	and, for any objects $k,l$ of $\mathcal{B}(n)$, the map
	\begin{equation}	
		\Hom_{\mathcal{B}(n)}(k,l) 
		\rightarrow 
		\Hom_{\mathcal{C}(Sp(n))}(\Phi(k),\Phi(l))
	\end{equation}
	is the map
	\begin{equation} 
		X_{k,n}^l : B_k^l(n) \rightarrow 
		\Hom_{Sp(n)}((\mathbb{R}^{n})^{\otimes k}, (\mathbb{R}^{n})^{\otimes l})
	\end{equation}
	given in Theorem \ref{spanningsetSp(n)}.
\end{theorem}

\begin{theorem} \label{brauerSO(n)functor}
	There exists a full, strict $\mathbb{R}$--linear monoidal functor
	\begin{equation}
		\Psi : \mathcal{BG}(n) \rightarrow \mathcal{C}(SO(n))
	\end{equation}
	that is defined on the objects of $\mathcal{BG}(n)$ by 
	$\Psi(k) \coloneqq ((\mathbb{R}^{n})^{\otimes k}, \rho_k)$
	and, for any objects $k,l$ of $\mathcal{B}(n)$, the map
	\begin{equation}	
		\Hom_{\mathcal{BG}(n)}(k,l) 
		\rightarrow 
		\Hom_{\mathcal{C}(SO(n))}(\Phi(k),\Phi(l))
	\end{equation}
	is the map
	\begin{equation} 
		\Psi_{k,n}^l : D_k^l(n) \rightarrow 
		\Hom_{SO(n)}((\mathbb{R}^{n})^{\otimes k}, (\mathbb{R}^{n})^{\otimes l})
	\end{equation}
	given in Theorem \ref{spanningsetSO(n)}.
\end{theorem}
Proofs of these results are given in the Technical Appendix. 


\section{Implications for Group Equivariant Neural Networks} \label{implicationsgroupequiv}


The fullness of the functors given in Section \ref{fullstrictmonfunct} is especially important. 
This condition immediately implies that, to understand any $G$--equivariant linear map in 
	$\Hom_{G}((\mathbb{R}^{n})^{\otimes k}, (\mathbb{R}^{n})^{\otimes l})$,
it is enough to work with the subset of $(k,l)$--partition diagrams that correspond to $G$, since we can apply the appropriate functor to obtain the equivariant maps themselves.

Furthermore, as the $(k,l)$--partition diagrams have a string-like aspect to them -- because they are morphisms in a strict $\mathbb{R}$--linear monoidal category -- we are able to drag and bend the strings and/or move the vertices to obtain new partition diagrams, and, in the process, new $G$--equivariant linear maps via the appropriate functor!

One very powerful use of this idea can be seen in Figure \ref{symmfactoring}. 
On the left hand side, we have an arbitrary $(5,3)$--partition diagram.
Suppose that we wish to multiply a vector $v \in
(\mathbb{R}^{n})^{\otimes 5}$, expressed in the standard basis,
by the matrix that the diagram corresponds to under the functor $\Theta$ given in Theorem \ref{partfunctor}.
We assume here that $n \geq 4$, since the number of blocks in the set partition corresponding to the $(5,3)$--partition diagram is $4$.
One option would be to multiply the vector by the matrix as given.
However, we can vastly improve the speed of the computation by performing a number of deformations to the diagram as shown in the figure.

At each stage, 
we drag and bend the strings representing the connected components of the set partition to obtain a factoring of the original $(5,3)$--partition diagram in terms of a composition of three other diagrams: 
a $(5,5)$--partition diagram that is not only Brauer but also a diagram that represents a permutation in the symmetric group $S_5$; 
another $(5,3)$--partition diagram that is \textit{planar} -- that is, none of the connected components in the diagram intersect each other --
and, finally, a diagram representing another permutation, this time in the symmetric group $S_3$.

Since the middle diagram is planar, this means that it can be decomposed as a tensor product of a number of simpler diagrams, using the strict monoidal property of $\mathcal{P}(n)$. 
The tensor product decomposition is shown in Figure \ref{tensorproddecomp}.
The key point is that, under the functor $\Theta$, the image of the planar partition diagram will be the Kronecker product of the images of these simpler diagrams, since the functor is strict $\mathbb{R}$--linear monoidal.
This will make the multiplication of any vector by this matrix
significantly quicker to perform.

Hence, we have outlined a procedure which shows that, to multiply $v$
by the linear map that is the image of the left hand diagram under $\Theta$ \textit{quickly},
we can first apply a permutation to the indices of the standard basis elements appearing in the input vector $v$, then apply the Kronecker product matrix that is the image under $\Theta$ of the planar $(5,3)$--partition diagram, and then finally apply another permutation to the indices of the standard basis elements in $(\mathbb{R}^{n})^{\otimes 3}$ appearing in the resulting vector.

By generalising this example to any $(k,l)$--partition diagram, we will recover
-- with one key distinction --
the algorithm of \cite{godfrey}
for applying symmetric group equivariant layer functions on tensor power spaces of $\mathbb{R}^{n}$ to input vectors; however, we have used a very different approach to obtain it.
The key distinction between the two versions
comes from making 
the middle diagram in the composition planar.
Moreover, it is not hard to see that this idea will generalise to give 
an algorithm for applying group equivariant linear maps to input vectors 
for the other groups presented in this paper.
This result will appear in another paper by the same authors.

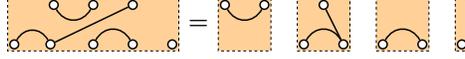
\begin{figure}
	\centering
	\scalebox{0.35}{\begin{tikzpicture}
	\begin{pgfonlayer}{nodelayer}
		\node [style=node] (0) at (12, -3) {};
		\node [style=node] (1) at (9, -3) {};
		\node [style=node] (2) at (6, -3) {};
		\node [style=node] (3) at (2.75, -3) {};
		\node [style=node] (4) at (6, 0) {};
		\node [style=node] (5) at (3, 0) {};
		\node [style=node] (6) at (9, 0) {};
		\node [style=node] (7) at (0, -3) {};
		\node [style=none] (10) at (-0.5, 0.5) {};
		\node [style=none] (11) at (-0.5, -3.5) {};
		\node [style=none] (12) at (12.5, -3.5) {};
		\node [style=none] (13) at (12.5, 0.5) {};
		\node [style=none] (14) at (14, -1.5) {\Huge$\bm{=}$};
		\node [style=node] (15) at (34, -3) {};
		\node [style=node] (16) at (31, -3) {};
		\node [style=node] (17) at (28, -3) {};
		\node [style=node] (18) at (25, -3) {};
		\node [style=node] (19) at (19, 0) {};
		\node [style=node] (20) at (16, 0) {};
		\node [style=node] (21) at (23.5, 0) {};
		\node [style=node] (22) at (22, -3) {};
		\node [style=none] (23) at (33.5, -3.5) {};
		\node [style=none] (24) at (33.5, 0.5) {};
		\node [style=none] (25) at (34.5, 0.5) {};
		\node [style=none] (26) at (34.5, -3.5) {};
		\node [style=none] (27) at (31.5, -3.5) {};
		\node [style=none] (28) at (27.5, -3.5) {};
		\node [style=none] (29) at (31.5, 0.5) {};
		\node [style=none] (30) at (27.5, 0.5) {};
		\node [style=none] (31) at (25.5, -3.5) {};
		\node [style=none] (32) at (21.5, -3.5) {};
		\node [style=none] (33) at (25.5, 0.5) {};
		\node [style=none] (34) at (21.5, 0.5) {};
		\node [style=none] (35) at (19.5, -3.5) {};
		\node [style=none] (36) at (15.5, -3.5) {};
		\node [style=none] (37) at (19.5, 0.5) {};
		\node [style=none] (38) at (15.5, 0.5) {};
	\end{pgfonlayer}
	\begin{pgfonlayer}{edgelayer}
		\draw [style={dash_4}] (13.center)
			 to (10.center)
			 to (11.center)
			 to (12.center)
			 to cycle;
		\draw [style=thick] (3) to (6);
		\draw [style=thick, bend left=60, looseness=1.25] (7) to (3);
		\draw [style=thick, bend left=60, looseness=1.25] (2) to (1);
		\draw [style=thick, bend right=60, looseness=1.25] (5) to (4);
		\draw [style={dash_4}] (23.center)
			 to (24.center)
			 to (25.center)
			 to (26.center)
			 to cycle;
		\draw [style={dash_4}] (38.center)
			 to (36.center)
			 to (35.center)
			 to (37.center)
			 to cycle;
		\draw [style={dash_4}] (31.center)
			 to (32.center)
			 to (34.center)
			 to (33.center)
			 to cycle;
		\draw [style={dash_4}] (28.center)
			 to (27.center)
			 to (29.center)
			 to (30.center)
			 to cycle;
		\draw [style=thick, bend left=60, looseness=1.25] (17) to (16);
		\draw [style=thick] (18) to (21);
		\draw [style=thick, bend left=60, looseness=1.25] (22) to (18);
		\draw [style=thick, bend right=60, looseness=1.25] (20) to (19);
	\end{pgfonlayer}
\end{tikzpicture}}
	\caption{The decomposition of the planar $(5,3)$--partition diagram into a tensor product of simpler partition diagrams.}
	\label{tensorproddecomp}
\end{figure}

\section{Related Work}

\cite{maron2018} studied the classification of linear permutation equivariant and invariant neural network layers. 
They characterised the learnable, linear, permutation equivariant layer functions in 
$\Hom_{S_n}((\mathbb{R}^{n})^{\otimes k}, (\mathbb{R}^{n})^{\otimes l})$
for $n \geq k + l$, using the orbit basis. 
However, \cite{godfrey} discovered that using the diagram basis, first constructed by \cite{Jones} in the case $k = l$, is beneficial for permutation equivariant neural network computations. 
\cite{pearcecrump} established the connection between permutation equivariant linear layers and the partition algebra, using Schur–Weyl duality. They fully characterised these layer functions for all values of $n$ and tensor power space orders, revealing that the dimension of the layer function space is not independent of $n$. 
The partition algebra was introduced by \cite{Martin0, Martin1, Martin2} and expanded upon by \cite{Jones}.
Recent papers by \cite{BenHal1, BenHal2} and~\cite{BHH}
show how the partition algebra can be used to construct the invariant theory of the symmetric group. 
\cite{comes} discovered the partition category and expressed it in terms of generators and relations.
\cite{pearcecrumpB} characterised all learnable, linear, equivariant layer functions in 
	$\Hom_{G}((\mathbb{R}^{n})^{\otimes k}, (\mathbb{R}^{n})^{\otimes l})$
for $G = O(n)$, $Sp(n)$, and $SO(n)$, using various sets of set partition diagrams. 
This characterisation was adapted from the combinatorial representation theory of the Brauer algebra, first developed by \cite{Brauer}.
\cite{grood} studied the representation theory of the Brauer–Grood algebra, while the Brauer category first appeared in \cite{LehrerZhang}.
The same authors investigated the theory behind what we have termed the Brauer–Grood category in \cite{LehrerZhang2}.

\section{Conclusion}

In this paper, we showed how category theory can be applied to the linear layer functions of group equivariant neural networks for the groups $S_n$, $O(n)$, $Sp(n)$, and $SO(n)$, resulting in a richer structure 
and a deeper understanding of these layer functions. 
In particular, we outlined the development of an algorithm for computing the result of a vector that is passed through an equivariant, linear layer for each group. 
The success of our approach suggests that category theory could be beneficial
for other areas of deep learning, leading to new insights and approaches.

\section{Acknowledgments}
The author would like to thank his PhD supervisor Professor William J. Knottenbelt for being generous with his time throughout the author's period of research prior to the publication of this paper.

This work was funded by the Doctoral Scholarship for Applied Research which was awarded to the author under Imperial College London's Department of Computing Applied Research scheme.
This work will form part of the author's PhD thesis at Imperial College London.
 


\nocite{*}
\bibliography{references}
\bibliographystyle{icml2023}


\newpage

\begin{appendix}
	
	\section{Composition and Tensor Product Operations for $D_k^l(n)$}

	Recall that $D_k^l(n)$ is defined to be the $\mathbb{R}$--linear span of the set of all $(k,l)$--Brauer diagrams together with the set of all $(l+k) \backslash n$--diagrams.

	We wish to define two $\mathbb{R}$-bilinear operations 
	\begin{align}
		\text{composition:   } & \bullet: D_l^m(n) \times D_k^l(n) \rightarrow D_k^m(n) \label{brauergroodcomposition} \\
		\text{tensor product:   } & \otimes: D_k^l(n) \times D_q^m(n) \rightarrow D_{k+q}^{l+m}(n) \label{brauergroodtensorprod}
\end{align}
	It is clear that we need to define these operations on each possible pair of diagrams, hence there are four cases (with the number of vertices to be amended depending on the operation):	
	\begin{enumerate}
		\item Brauer and Brauer
		\item Brauer and $(l+k) \backslash n$
		\item $(l+k) \backslash n$ and Brauer, and
		\item $(l+k) \backslash n$ and $(l+k) \backslash n$
	\end{enumerate}
	To do this, we follow the constructions given in \cite{comes}. 
	\cite[Section 3.1]{comes}, in particular, contains excellent motivation as to why we do the following.

	Firstly, we introduce a \textit{jellyfish}, which takes the $n$ free vertices in a $(l+k) \backslash n$ diagram
	and attaches them to a blue head, in top--to--bottom, left--to-right order.
	For example, if $k = 2$, $l = 1$ and $n = 3$,
	the jellyfish has the form
	\begin{equation} \label{attachjellyfish}
		\begin{aligned}
		\scalebox{0.5}{\begin{tikzpicture}
	\begin{pgfonlayer}{nodelayer}
		\node [style=none] (54) at (6, 0) {};
		\node [style=none] (55) at (6, -4) {};
		\node [style=none] (56) at (13, -4) {};
		\node [style=none] (57) at (13, 0) {};
		\node [style=node] (58) at (9.5, -3.5) {};
		\node [style=node] (59) at (8, -0.5) {};
		\node [style=none] (60) at (9.5, -1.75) {};
		\node [style=none] (61) at (12.5, -1.75) {};
		\node [style=none] (62) at (11, -0.5) {};
		\node [style=none] (63) at (11, -1.75) {};
		\node [style=none] (64) at (10.25, -1.75) {};
		\node [style=none] (65) at (11.75, -1.75) {};
		\node [style=none] (66) at (5, -2) {$\bm{\rightarrow}$};
		\node [style=node] (67) at (6.5, -3.5) {};
		\node [style=none] (68) at (0, 0) {};
		\node [style=none] (69) at (0, -4) {};
		\node [style=none] (70) at (4, -4) {};
		\node [style=none] (71) at (4, 0) {};
		\node [style=node] (72) at (3.5, -3.5) {};
		\node [style=node] (73) at (2, -0.5) {};
		\node [style=node] (80) at (0.5, -3.5) {};
	\end{pgfonlayer}
	\begin{pgfonlayer}{edgelayer}
		\draw [style={dash_2}] (56.center) to (57.center);
		\draw [style={dash_2}] (57.center) to (54.center);
		\draw [style={dash_2}] (54.center) to (55.center);
		\draw [style={dash_2}] (55.center) to (56.center);
		\draw [style=jellyfish] (60.center)
			 to [bend left=45, looseness=0.75] (62.center)
			 to [bend left=45, looseness=0.75] (61.center)
			 to cycle;
		\draw [style=thick, bend right=75, looseness=1.75] (59) to (64.center);
		\draw [style={dash_2}] (70.center) to (71.center);
		\draw [style={dash_2}] (71.center) to (68.center);
		\draw [style={dash_2}] (68.center) to (69.center);
		\draw [style={dash_2}] (69.center) to (70.center);
		\draw [style=thick, in=-120, out=-15, looseness=0.50] (67) to (63.center);
		\draw [style=thick] (58) to (65.center);
	\end{pgfonlayer}
\end{tikzpicture}}
		\end{aligned}	
	\end{equation}
	The jellyfish satisfies two rules.

	\textbf{Rule 1}: composing an adjacent permutation with a jellyfish, which is equivalent to having two adjacent legs of the jellyfish crossed, is equal to the negative of the uncrossed jellyfish:
	\begin{center}
		\scalebox{0.5}{\begin{tikzpicture}
	\begin{pgfonlayer}{nodelayer}
		\node [style=none] (29) at (15.75, -2) {$\bm{=}$};
		\node [style=node] (30) at (21.75, -3.5) {};
		\node [style=node] (31) at (19.75, -3.5) {};
		\node [style=node] (32) at (18.75, -3.5) {};
		\node [style=none] (33) at (18.75, -1.75) {};
		\node [style=none] (34) at (21.75, -1.75) {};
		\node [style=none] (35) at (20.25, -0.5) {};
		\node [style=none] (36) at (20.25, -1.75) {};
		\node [style=none] (37) at (19.5, -1.75) {};
		\node [style=none] (38) at (21, -1.75) {};
		\node [style=none] (39) at (18.25, 0) {};
		\node [style=none] (40) at (18.25, -4) {};
		\node [style=none] (41) at (22.25, -4) {};
		\node [style=none] (42) at (22.25, 0) {};
		\node [style=none] (43) at (20.75, -3.5) {$\bm{\ldots}$};
		\node [style=none] (44) at (17.25, -2) {\LARGE$\bm{-}$};
		\node [style=none] (45) at (4.5, -2) {};
		\node [style=none] (46) at (9.5, -2) {};
		\node [style=none] (47) at (7, 0) {};
		\node [style=none] (48) at (-0.5, 0.5) {};
		\node [style=none] (49) at (-0.5, -4.5) {};
		\node [style=none] (50) at (14.75, -4.5) {};
		\node [style=none] (51) at (14.75, 0.5) {};
		\node [style=node] (60) at (10, -4) {};
		\node [style=node] (61) at (6, -4) {};
		\node [style=node] (62) at (4, -4) {};
		\node [style=none] (63) at (12, -4) {$\bm{\ldots}$};
		\node [style=node] (64) at (8, -4) {};
		\node [style=node] (65) at (14.25, -4) {};
		\node [style=node] (66) at (0, -4) {};
		\node [style=none] (67) at (2, -4) {$\bm{\ldots}$};
		\node [style=none] (68) at (5, -2) {};
		\node [style=none] (69) at (5.75, -2) {};
		\node [style=none] (70) at (9, -2) {};
		\node [style=none] (71) at (8.25, -2) {};
		\node [style=none] (72) at (7.5, -2) {};
		\node [style=none] (73) at (6.5, -2) {};
		\node [style=none] (74) at (-12.75, -1) {};
		\node [style=none] (75) at (-7.75, -1) {};
		\node [style=none] (76) at (-10.25, 1) {};
		\node [style=none] (77) at (-17.75, 1.5) {};
		\node [style=none] (78) at (-17.75, -5.5) {};
		\node [style=none] (79) at (-2.5, -5.5) {};
		\node [style=none] (80) at (-2.5, 1.5) {};
		\node [style=node] (81) at (-7.25, -5) {};
		\node [style=node] (82) at (-11.25, -5) {};
		\node [style=node] (83) at (-13.25, -5) {};
		\node [style=none] (84) at (-5.25, -5) {$\bm{\ldots}$};
		\node [style=node] (85) at (-9.25, -5) {};
		\node [style=node] (86) at (-3, -5) {};
		\node [style=node] (87) at (-17.25, -5) {};
		\node [style=none] (88) at (-15.25, -5) {$\bm{\ldots}$};
		\node [style=node] (89) at (-7.25, -3) {};
		\node [style=node] (90) at (-11.25, -3) {};
		\node [style=node] (91) at (-13.25, -3) {};
		\node [style=none] (92) at (-5.25, -3) {$\bm{\ldots}$};
		\node [style=node] (93) at (-9.25, -3) {};
		\node [style=node] (94) at (-3, -3) {};
		\node [style=node] (95) at (-17.25, -3) {};
		\node [style=none] (96) at (-15.25, -3) {$\bm{\ldots}$};
		\node [style=none] (97) at (-12.25, -1) {};
		\node [style=none] (98) at (-11.5, -1) {};
		\node [style=none] (99) at (-8.25, -1) {};
		\node [style=none] (100) at (-9, -1) {};
		\node [style=none] (101) at (-9.75, -1) {};
		\node [style=none] (102) at (-10.75, -1) {};
		\node [style=none] (103) at (-1.5, -2) {$\bm{=}$};
	\end{pgfonlayer}
	\begin{pgfonlayer}{edgelayer}
		\draw [style={dash_2}] (41.center) to (42.center);
		\draw [style={dash_2}] (42.center) to (39.center);
		\draw [style={dash_2}] (39.center) to (40.center);
		\draw [style={dash_2}] (40.center) to (41.center);
		\draw [style=jellyfish] (35.center)
			 to [bend left=45, looseness=0.75] (34.center)
			 to (33.center)
			 to [bend left=45, looseness=0.75] cycle;
		\draw [style=thick] (32) to (37.center);
		\draw [style=thick] (31) to (36.center);
		\draw [style=thick] (30) to (38.center);
		\draw [style={dash_2}] (48.center) to (49.center);
		\draw [style={dash_2}] (49.center) to (50.center);
		\draw [style={dash_2}] (50.center) to (51.center);
		\draw [style={dash_2}] (51.center) to (48.center);
		\draw [style=jellyfish] (47.center)
			 to [bend left=315, looseness=0.75] (45.center)
			 to (46.center)
			 to [bend right=45, looseness=0.75] cycle;
		\draw [style=thick] (66) to (68.center);
		\draw [style=thick] (62) to (69.center);
		\draw [style=thick] (60) to (71.center);
		\draw [style=thick] (65) to (70.center);
		\draw [style={dash_2}] (77.center) to (78.center);
		\draw [style={dash_2}] (78.center) to (79.center);
		\draw [style={dash_2}] (79.center) to (80.center);
		\draw [style={dash_2}] (80.center) to (77.center);
		\draw [style=jellyfish] (74.center)
			 to (75.center)
			 to [bend right=45, looseness=0.75] (76.center)
			 to [bend left=315, looseness=0.75] cycle;
		\draw [style=thick] (90) to (85);
		\draw [style=thick] (82) to (93);
		\draw [style=thick] (91) to (83);
		\draw [style=thick] (81) to (89);
		\draw [style=thick] (95) to (87);
		\draw [style=thick] (86) to (94);
		\draw [style=thick] (95) to (97.center);
		\draw [style=thick] (91) to (98.center);
		\draw [style=thick] (90) to (102.center);
		\draw [style=thick] (93) to (101.center);
		\draw [style=thick] (89) to (100.center);
		\draw [style=thick] (94) to (99.center);
		\draw [style=thick] (64) to (73.center);
		\draw [style=thick] (72.center) to (61);
	\end{pgfonlayer}
\end{tikzpicture}}
	\end{center}
	\textbf{Rule 2}: the juxtaposition of two jellyfish, each with $n$ legs, is equal to the following linear combination of Brauer diagrams.
	\begin{center}
		\scalebox{0.5}{\begin{tikzpicture}
	\begin{pgfonlayer}{nodelayer}
		\node [style=node] (0) at (3.5, -3.5) {};
		\node [style=node] (1) at (1.5, -3.5) {};
		\node [style=node] (2) at (0.5, -3.5) {};
		\node [style=none] (3) at (0.5, -1.75) {};
		\node [style=none] (4) at (3.5, -1.75) {};
		\node [style=none] (5) at (2, -0.5) {};
		\node [style=none] (6) at (2, -1.75) {};
		\node [style=none] (7) at (1.25, -1.75) {};
		\node [style=none] (8) at (2.75, -1.75) {};
		\node [style=none] (9) at (0, 0) {};
		\node [style=none] (10) at (0, -4) {};
		\node [style=none] (11) at (8, -4) {};
		\node [style=none] (12) at (8, 0) {};
		\node [style=none] (13) at (2.5, -3.5) {$\bm{\ldots}$};
		\node [style=node] (14) at (7.5, -3.5) {};
		\node [style=node] (15) at (5.5, -3.5) {};
		\node [style=node] (16) at (4.5, -3.5) {};
		\node [style=none] (17) at (4.5, -1.75) {};
		\node [style=none] (18) at (7.5, -1.75) {};
		\node [style=none] (19) at (6, -0.5) {};
		\node [style=none] (20) at (6, -1.75) {};
		\node [style=none] (21) at (5.25, -1.75) {};
		\node [style=none] (22) at (6.75, -1.75) {};
		\node [style=none] (27) at (6.5, -3.5) {$\bm{\ldots}$};
		\node [style=none] (28) at (9, -2) {$\bm{=}$};
		\node [style=none] (29) at (11.5, -2.5) {\LARGE$\bm{\sum \limits_{\sigma \in S_n}}$};
		\node [style=none] (33) at (16.75, 2.5) {};
		\node [style=none] (34) at (16.75, -6) {};
		\node [style=none] (35) at (32.5, -6) {};
		\node [style=none] (36) at (32.5, 2.5) {};
		\node [style=node] (37) at (27.75, -5.5) {};
		\node [style=node] (38) at (23.75, -5.5) {};
		\node [style=node] (39) at (19.75, -5.5) {};
		\node [style=none] (40) at (29.75, -5.5) {$\bm{\ldots}$};
		\node [style=node] (41) at (25.75, -5.5) {};
		\node [style=node] (42) at (32, -5.5) {};
		\node [style=node] (43) at (17.75, -5.5) {};
		\node [style=none] (44) at (21.75, -5.5) {$\bm{\ldots}$};
		\node [style=node] (45) at (27.75, -1) {};
		\node [style=node] (46) at (23.75, -1) {};
		\node [style=node] (47) at (19.75, -1) {};
		\node [style=none] (48) at (29.75, -1) {$\bm{\ldots}$};
		\node [style=node] (49) at (25.75, -1) {};
		\node [style=node] (50) at (32, -1) {};
		\node [style=node] (51) at (17.75, -1) {};
		\node [style=none] (52) at (21.75, -1) {$\bm{\ldots}$};
		\node [style=none] (53) at (17.25, -2.25) {};
		\node [style=none] (54) at (17.25, -4.25) {};
		\node [style=none] (55) at (24.25, -4.25) {};
		\node [style=none] (56) at (24.25, -2.25) {};
		\node [style=none] (57) at (20.75, -3.25) {\LARGE$\bm{\sigma}$};
		\node [style=none] (58) at (17.75, -2.25) {};
		\node [style=none] (59) at (19.75, -2.25) {};
		\node [style=none] (60) at (23.75, -2.25) {};
		\node [style=none] (61) at (23.75, -4.25) {};
		\node [style=none] (62) at (19.75, -4.25) {};
		\node [style=none] (63) at (17.75, -4.25) {};
		\node [style=none] (64) at (14.5, -2) {\LARGE$\bm{(-1)^{\sigma}}$};
	\end{pgfonlayer}
	\begin{pgfonlayer}{edgelayer}
		\draw [style={dash_2}] (11.center) to (12.center);
		\draw [style={dash_2}] (12.center) to (9.center);
		\draw [style={dash_2}] (9.center) to (10.center);
		\draw [style={dash_2}] (10.center) to (11.center);
		\draw [style=jellyfish] (5.center)
			 to [bend left=45, looseness=0.75] (4.center)
			 to (3.center)
			 to [bend left=45, looseness=0.75] cycle;
		\draw [style=thick] (2) to (7.center);
		\draw [style=thick] (1) to (6.center);
		\draw [style=thick] (0) to (8.center);
		\draw [style=jellyfish] (19.center)
			 to [bend left=45, looseness=0.75] (18.center)
			 to (17.center)
			 to [bend left=45, looseness=0.75] cycle;
		\draw [style=thick] (16) to (21.center);
		\draw [style=thick] (15) to (20.center);
		\draw [style=thick] (14) to (22.center);
		\draw [style={dash_2}] (33.center) to (34.center);
		\draw [style={dash_2}] (34.center) to (35.center);
		\draw [style={dash_2}] (35.center) to (36.center);
		\draw [style={dash_2}] (36.center) to (33.center);
		\draw [style=thick] (42) to (50);
		\draw [style=thick] (53.center) to (56.center);
		\draw [style=thick] (56.center) to (55.center);
		\draw [style=thick] (55.center) to (54.center);
		\draw [style=thick] (54.center) to (53.center);
		\draw [style=thick] (58.center) to (51);
		\draw [style=thick] (43) to (63.center);
		\draw [style=thick] (39) to (62.center);
		\draw [style=thick] (47) to (59.center);
		\draw [style=thick] (60.center) to (46);
		\draw [style=thick] (38) to (61.center);
		\draw [style=thick] (41) to (49);
		\draw [style=thick] (37) to (45);
		\draw [style=thick, bend left=90] (46) to (50);
		\draw [style=thick, bend left=90] (47) to (45);
		\draw [style=thick, bend left=90] (51) to (49);
	\end{pgfonlayer}
\end{tikzpicture}}
	\end{center}
	For example, if $n = 2$, then we have that
	\begin{center}
		\scalebox{0.5}{\begin{tikzpicture}
	\begin{pgfonlayer}{nodelayer}
		\node [style=node] (0) at (3.5, -3.5) {};
		\node [style=node] (2) at (0.5, -3.5) {};
		\node [style=none] (3) at (0.5, -1.75) {};
		\node [style=none] (4) at (3.5, -1.75) {};
		\node [style=none] (5) at (2, -0.5) {};
		\node [style=none] (6) at (2, -1.75) {};
		\node [style=none] (7) at (1.25, -1.75) {};
		\node [style=none] (8) at (2.75, -1.75) {};
		\node [style=none] (9) at (0, 0) {};
		\node [style=none] (10) at (0, -4) {};
		\node [style=none] (11) at (8, -4) {};
		\node [style=none] (12) at (8, 0) {};
		\node [style=node] (14) at (7.5, -3.5) {};
		\node [style=node] (16) at (4.5, -3.5) {};
		\node [style=none] (17) at (4.5, -1.75) {};
		\node [style=none] (18) at (7.5, -1.75) {};
		\node [style=none] (19) at (6, -0.5) {};
		\node [style=none] (20) at (6, -1.75) {};
		\node [style=none] (21) at (5.25, -1.75) {};
		\node [style=none] (22) at (6.75, -1.75) {};
		\node [style=none] (28) at (9, -2) {$\bm{=}$};
		\node [style=node] (29) at (12.75, -3.5) {};
		\node [style=node] (30) at (10.5, -3.5) {};
		\node [style=none] (37) at (10, 0) {};
		\node [style=none] (38) at (10, -4) {};
		\node [style=none] (39) at (18, -4) {};
		\node [style=none] (40) at (18, 0) {};
		\node [style=node] (41) at (17.5, -3.5) {};
		\node [style=node] (42) at (15.25, -3.5) {};
		\node [style=node] (49) at (22.75, -3.5) {};
		\node [style=node] (50) at (20.5, -3.5) {};
		\node [style=none] (57) at (20, 0) {};
		\node [style=none] (58) at (20, -4) {};
		\node [style=none] (59) at (28, -4) {};
		\node [style=none] (60) at (28, 0) {};
		\node [style=node] (61) at (27.5, -3.5) {};
		\node [style=node] (62) at (25.25, -3.5) {};
		\node [style=none] (69) at (19, -2) {$\bm{-}$};
	\end{pgfonlayer}
	\begin{pgfonlayer}{edgelayer}
		\draw [style={dash_2}] (11.center) to (12.center);
		\draw [style={dash_2}] (12.center) to (9.center);
		\draw [style={dash_2}] (9.center) to (10.center);
		\draw [style={dash_2}] (10.center) to (11.center);
		\draw [style=jellyfish] (5.center)
			 to [bend left=45, looseness=0.75] (4.center)
			 to (3.center)
			 to [bend left=45, looseness=0.75] cycle;
		\draw [style=thick] (2) to (7.center);
		\draw [style=thick] (0) to (8.center);
		\draw [style=jellyfish] (19.center)
			 to [bend left=45, looseness=0.75] (18.center)
			 to (17.center)
			 to [bend left=45, looseness=0.75] cycle;
		\draw [style=thick] (16) to (21.center);
		\draw [style=thick] (14) to (22.center);
		\draw [style={dash_2}] (39.center) to (40.center);
		\draw [style={dash_2}] (40.center) to (37.center);
		\draw [style={dash_2}] (37.center) to (38.center);
		\draw [style={dash_2}] (38.center) to (39.center);
		\draw [style={dash_2}] (59.center) to (60.center);
		\draw [style={dash_2}] (60.center) to (57.center);
		\draw [style={dash_2}] (57.center) to (58.center);
		\draw [style={dash_2}] (58.center) to (59.center);
		\draw [style=thick, bend left=90, looseness=2.00] (30) to (42);
		\draw [style=thick, bend left=90, looseness=2.00] (29) to (41);
		\draw [style=thick, bend left=90, looseness=1.25] (50) to (61);
		\draw [style=thick, bend left=90, looseness=2.00] (49) to (62);
	\end{pgfonlayer}
\end{tikzpicture}}
	\end{center}
	One consequence of Rule 1 that we need is that if two legs of the same jellyfish are connected, then it is equal to $0$. For example
	\begin{center}
		\scalebox{0.5}{\begin{tikzpicture}
	\begin{pgfonlayer}{nodelayer}
		\node [style=none] (8) at (4, 0) {};
		\node [style=none] (9) at (4, -4) {};
		\node [style=none] (10) at (8, -4) {};
		\node [style=none] (11) at (8, 0) {};
		\node [style=node] (12) at (7.5, -3.5) {};
		\node [style=node] (13) at (4.5, -3.5) {};
		\node [style=none] (14) at (4.5, -1.75) {};
		\node [style=none] (15) at (7.5, -1.75) {};
		\node [style=none] (16) at (6, -0.5) {};
		\node [style=none] (17) at (6, -1.75) {};
		\node [style=none] (18) at (5.25, -1.75) {};
		\node [style=none] (19) at (6.75, -1.75) {};
		\node [style=none] (20) at (9, -2) {$\bm{=}$};
		\node [style=node] (21) at (6, -3.5) {};
		\node [style=none] (22) at (3, -2) {\Large$\bm{2}$};
		\node [style=none] (23) at (10, 0) {};
		\node [style=none] (24) at (10, -4) {};
		\node [style=none] (25) at (14, -4) {};
		\node [style=none] (26) at (14, 0) {};
		\node [style=node] (27) at (13.5, -3.5) {};
		\node [style=node] (28) at (10.5, -3.5) {};
		\node [style=none] (29) at (10.5, -1.75) {};
		\node [style=none] (30) at (13.5, -1.75) {};
		\node [style=none] (31) at (12, -0.5) {};
		\node [style=none] (32) at (12, -1.75) {};
		\node [style=none] (33) at (11.25, -1.75) {};
		\node [style=none] (34) at (12.75, -1.75) {};
		\node [style=none] (35) at (15, -2) {\Large$\bm{+}$};
		\node [style=node] (36) at (12, -3.5) {};
		\node [style=none] (37) at (16, 0) {};
		\node [style=none] (38) at (16, -4) {};
		\node [style=none] (39) at (20, -4) {};
		\node [style=none] (40) at (20, 0) {};
		\node [style=node] (41) at (19.5, -3.5) {};
		\node [style=node] (42) at (16.5, -3.5) {};
		\node [style=none] (43) at (16.5, -1.75) {};
		\node [style=none] (44) at (19.5, -1.75) {};
		\node [style=none] (45) at (18, -0.5) {};
		\node [style=none] (46) at (18, -1.75) {};
		\node [style=none] (47) at (17.25, -1.75) {};
		\node [style=none] (48) at (18.75, -1.75) {};
		\node [style=node] (49) at (18, -3.5) {};
		\node [style=none] (50) at (21, -2) {$\bm{=}$};
		\node [style=none] (51) at (22, 0) {};
		\node [style=none] (52) at (22, -4) {};
		\node [style=none] (53) at (26, -4) {};
		\node [style=none] (54) at (26, 0) {};
		\node [style=node] (55) at (25.5, -3.5) {};
		\node [style=node] (56) at (22.5, -3.5) {};
		\node [style=none] (57) at (22.5, -1.75) {};
		\node [style=none] (58) at (25.5, -1.75) {};
		\node [style=none] (59) at (24, -0.5) {};
		\node [style=none] (60) at (24, -1.75) {};
		\node [style=none] (61) at (23.25, -1.75) {};
		\node [style=none] (62) at (24.75, -1.75) {};
		\node [style=none] (63) at (27, -2) {\Large$\bm{-}$};
		\node [style=node] (64) at (24, -3.5) {};
		\node [style=none] (65) at (28, 0) {};
		\node [style=none] (66) at (28, -4) {};
		\node [style=none] (67) at (32, -4) {};
		\node [style=none] (68) at (32, 0) {};
		\node [style=node] (69) at (31.5, -3.5) {};
		\node [style=node] (70) at (28.5, -3.5) {};
		\node [style=none] (71) at (28.5, -1.75) {};
		\node [style=none] (72) at (31.5, -1.75) {};
		\node [style=none] (73) at (30, -0.5) {};
		\node [style=none] (74) at (30, -1.75) {};
		\node [style=none] (75) at (29.25, -1.75) {};
		\node [style=none] (76) at (30.75, -1.75) {};
		\node [style=node] (77) at (30, -3.5) {};
		\node [style=none] (78) at (33, -2) {$\bm{=}$};
		\node [style=none] (79) at (34, -2) {\Large$\bm{0}$};
	\end{pgfonlayer}
	\begin{pgfonlayer}{edgelayer}
		\draw [style={dash_2}] (10.center) to (11.center);
		\draw [style={dash_2}] (11.center) to (8.center);
		\draw [style={dash_2}] (8.center) to (9.center);
		\draw [style={dash_2}] (9.center) to (10.center);
		\draw [style=jellyfish] (16.center)
			 to [bend left=45, looseness=0.75] (15.center)
			 to (14.center)
			 to [bend left=45, looseness=0.75] cycle;
		\draw [style=thick] (13) to (18.center);
		\draw [style=thick] (12) to (19.center);
		\draw [style=thick] (21) to (17.center);
		\draw [style=thick] (21) to (12);
		\draw [style={dash_2}] (25.center) to (26.center);
		\draw [style={dash_2}] (26.center) to (23.center);
		\draw [style={dash_2}] (23.center) to (24.center);
		\draw [style={dash_2}] (24.center) to (25.center);
		\draw [style=jellyfish] (29.center)
			 to [bend left=45, looseness=0.75] (31.center)
			 to [bend left=45, looseness=0.75] (30.center)
			 to cycle;
		\draw [style=thick] (28) to (33.center);
		\draw [style=thick] (36) to (27);
		\draw [style={dash_2}] (39.center) to (40.center);
		\draw [style={dash_2}] (40.center) to (37.center);
		\draw [style={dash_2}] (37.center) to (38.center);
		\draw [style={dash_2}] (38.center) to (39.center);
		\draw [style=jellyfish] (43.center)
			 to [bend left=45, looseness=0.75] (45.center)
			 to [bend left=45, looseness=0.75] (44.center)
			 to cycle;
		\draw [style=thick] (42) to (47.center);
		\draw [style=thick] (41) to (48.center);
		\draw [style=thick] (49) to (46.center);
		\draw [style=thick] (49) to (41);
		\draw [style=thick] (36) to (34.center);
		\draw [style=thick] (27) to (32.center);
		\draw [style={dash_2}] (53.center) to (54.center);
		\draw [style={dash_2}] (54.center) to (51.center);
		\draw [style={dash_2}] (51.center) to (52.center);
		\draw [style={dash_2}] (52.center) to (53.center);
		\draw [style=jellyfish] (57.center)
			 to [bend left=45, looseness=0.75] (59.center)
			 to [bend left=45, looseness=0.75] (58.center)
			 to cycle;
		\draw [style=thick] (56) to (61.center);
		\draw [style=thick] (64) to (55);
		\draw [style={dash_2}] (67.center) to (68.center);
		\draw [style={dash_2}] (68.center) to (65.center);
		\draw [style={dash_2}] (65.center) to (66.center);
		\draw [style={dash_2}] (66.center) to (67.center);
		\draw [style=jellyfish] (71.center)
			 to [bend left=45, looseness=0.75] (73.center)
			 to [bend left=45, looseness=0.75] (72.center)
			 to cycle;
		\draw [style=thick] (70) to (75.center);
		\draw [style=thick] (77) to (69);
		\draw [style=thick] (64) to (62.center);
		\draw [style=thick] (55) to (60.center);
		\draw [style=thick] (77) to (76.center);
		\draw [style=thick] (69) to (74.center);
	\end{pgfonlayer}
\end{tikzpicture}}
	\end{center}
	With this, we can define the composition and tensor product operations for each of the cases.

	\subsection{Composition}

	\textbf{Case 1}: The composition of a $(k,l)$--Brauer diagram with a $(l,m)$--Brauer diagram is the usual composition as defined in (\ref{composition}).

	\textbf{Cases 2, 3}: We explain this in the situation where we have a $(k,l)$--Brauer diagram $d_\beta$ and a $(m+l) \backslash n$--diagram $d_\alpha$, but the same steps apply in the other case.
	We perform the following procedure:
	\begin{itemize}
		\item Attach a jellyfish to $d_\alpha$ as in (\ref{attachjellyfish}).
		\item Place $d_\beta$ below $d_\alpha$, and concatenate the diagrams as per usual to obtain a diagram with three rows of vertices.
		\item If any two legs of the same jellyfish are connected in this diagram, the result is $0$, and we can stop.
		\item Otherwise, fully concatenate the diagrams to obtain a diagram with two rows of vertices, removing all connected components that lie entirely in the middle row of the concatenated diagrams.
		Let $c(d_\alpha, d_\beta)$ be the number of connected components that are removed from the middle row.
		\item Now repeatedly apply Rule 1 to fully uncross the legs of the jellyfish.  
			This action is actually a permutation $\sigma \in S_n$, and the resulting sign in front of the diagram that we obtain is $\sgn(\sigma)$.
		\item Finally, detach/remove the jellyfish to obtain a $(m+k) \backslash n$--diagram $d_\gamma$.
			Consequently, the result that we obtain is $\sgn(\sigma)n^{c(d_\alpha, d_\beta)}d_\gamma$.
	\end{itemize}

	\textbf{Case 4}: Suppose that we have a 
	$(l+k) \backslash n$--diagram $d_{\alpha_1}$ and a
	$(m+l) \backslash n$--diagram $d_{\alpha_2}$.
	We perform the following procedure:
	\begin{itemize}
		\item Attach separate jellyfish to $d_{\alpha_1}$ and $d_{\alpha_2}$ as in (\ref{attachjellyfish}).
		\item Place $d_{\alpha_1}$ below $d_{\alpha_2}$, and concatenate the diagrams as per usual to obtain a diagram with three rows of vertices.
		\item If any two legs of the same jellyfish are connected in this diagram, the result is $0$, and we can stop.
		\item Otherwise, fully concatenate the diagrams to obtain a diagram with two rows of vertices, removing all connected components that lie entirely in the middle row of the concatenated diagrams.
			Let $c(d_{\alpha_2}, d_{\alpha_1})$ be the number of connected components that are removed from the middle row.
		\item Now repeatedly apply rule 1 to fully uncross the legs of each of the jellyfish. 
			This action is actually a product of permutations $\sigma\tau \in S_n$, and the resulting sign of the diagram that we obtain is $\sgn(\sigma\tau) = \sgn(\sigma)\sgn(\tau)$.
		\item Now apply Rule 2 to obtain a linear combination of $(k,m)$--Brauer diagrams. Call this linear combination $d$.
			Consequently, the result that we obtain is $\sgn(\sigma)\sgn(\tau)n^{c(d_{\alpha_2}, d_{\alpha_1})}d$.
	\end{itemize}

	\subsection{Tensor Product}

	\textbf{Case 1}: The tensor product of a $(k,l)$--Brauer diagram with a $(q,m)$--Brauer diagram is the usual tensor product as defined in (\ref{tensorprod}).
	
	\textbf{Case 2}: 
	The tensor product of a $(k,l)$--Brauer diagram $d_\beta$ with an $(m+q) \backslash n$--diagram $d_\alpha$ is defined to be the $(l+m+k+q) \backslash n$--diagram that is obtained by horizontally placing $d_\beta$ to the left of $d_\alpha$ without any overlapping of vertices.

	\textbf{Case 3}:
	The tensor product of a 
	$(l+k) \backslash n$--diagram $d_\alpha$ 
	with a
	$(q,m)$--Brauer diagram $d_\beta$ 
	is defined to be the $(l+m+k+q) \backslash n$--diagram that is obtained by horizontally placing $d_\alpha$ to the left of $d_\beta$ without any overlapping of vertices.

	\textbf{Case 4}: Suppose that we have a 
	$(l+k) \backslash n$--diagram $d_{\alpha_1}$ and a
	$(m+q) \backslash n$--diagram $d_{\alpha_2}$.
	We perform the following procedure:
	\begin{itemize}
		\item Attach separate jellyfish to $d_{\alpha_1}$ and $d_{\alpha_2}$ as in (\ref{attachjellyfish}).
		\item Horizontally place $d_{\alpha_1}$ to the left of $d_{\alpha_2}$ without any overlapping of vertices.
		\item Apply Rule 2 to obtain a linear combination of $(k+q, l+m)$--Brauer diagrams. This is the end result.
	\end{itemize}

	\subsection{Examples}

	\textbf{Composition, Case 3}	

	Let $d_\alpha$ be the diagram
	\begin{center}
		\scalebox{0.5}
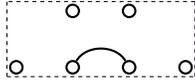
{\begin{tikzpicture}
	\begin{pgfonlayer}{nodelayer}
		\node [style=none] (0) at (0.5, 0) {};
		\node [style=none] (1) at (0.5, -4) {};
		\node [style=node] (5) at (4, -0.5) {};
		\node [style=node] (6) at (1, -3.5) {};
		\node [style=node] (12) at (7, -0.5) {};
		\node [style=none] (16) at (10.5, -4) {};
		\node [style=none] (17) at (10.5, 0) {};
		\node [style=node] (18) at (10, -3.5) {};
		\node [style=node] (19) at (4, -3.5) {};
		\node [style=node] (20) at (7, -3.5) {};
	\end{pgfonlayer}
	\begin{pgfonlayer}{edgelayer}
		\draw [style={dash_2}] (0.center) to (1.center);
		\draw [style={dash_2}] (1.center) to (16.center);
		\draw [style={dash_2}] (16.center) to (17.center);
		\draw [style={dash_2}] (17.center) to (0.center);
		\draw [style=thick, bend left=60] (19) to (20);
	\end{pgfonlayer}
\end{tikzpicture}}
	\end{center}
	and let $d_\beta$ be the diagram 
	\begin{center}
		\scalebox{0.5}
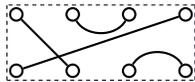
{\begin{tikzpicture}
	\begin{pgfonlayer}{nodelayer}
		\node [style=none] (0) at (0.5, 0) {};
		\node [style=none] (1) at (0.5, -4) {};
		\node [style=node] (4) at (4, -3.5) {};
		\node [style=node] (5) at (4, -0.5) {};
		\node [style=node] (6) at (1, -3.5) {};
		\node [style=node] (11) at (7, -3.5) {};
		\node [style=node] (12) at (7, -0.5) {};
		\node [style=none] (16) at (10.5, -4) {};
		\node [style=none] (17) at (10.5, 0) {};
		\node [style=node] (18) at (10, -3.5) {};
		\node [style=node] (19) at (1, -0.5) {};
		\node [style=node] (20) at (10, -0.5) {};
	\end{pgfonlayer}
	\begin{pgfonlayer}{edgelayer}
		\draw [style={dash_2}] (0.center) to (1.center);
		\draw [style={dash_2}] (1.center) to (16.center);
		\draw [style={dash_2}] (16.center) to (17.center);
		\draw [style={dash_2}] (17.center) to (0.center);
		\draw [style=thick, bend right=60] (5) to (12);
		\draw [style=thick, bend left=60] (11) to (18);
		\draw [style=thick] (4) to (19);
		\draw [style=thick] (6) to (20);
	\end{pgfonlayer}
\end{tikzpicture}}
	\end{center}
	Then $d_\alpha \bullet d_\beta$ is equal to
	\begin{center}
		\scalebox{0.5}
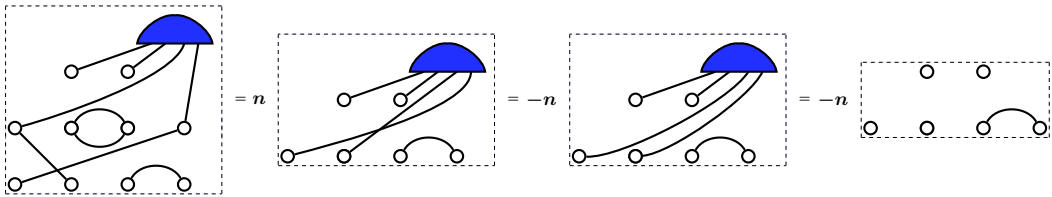
{\begin{tikzpicture}
	\begin{pgfonlayer}{nodelayer}
		\node [style=none] (0) at (0, 0) {};
		\node [style=none] (1) at (0, -4) {};
		\node [style=node] (4) at (3.5, -3.5) {};
		\node [style=node] (5) at (3.5, -0.5) {};
		\node [style=node] (6) at (0.5, -3.5) {};
		\node [style=node] (11) at (6.5, -3.5) {};
		\node [style=node] (12) at (6.5, -0.5) {};
		\node [style=none] (16) at (11.5, -4) {};
		\node [style=none] (17) at (11.5, 0) {};
		\node [style=node] (18) at (9.5, -3.5) {};
		\node [style=node] (19) at (0.5, -0.5) {};
		\node [style=node] (20) at (9.5, -0.5) {};
		\node [style=none] (21) at (0, 6) {};
		\node [style=none] (22) at (0, -1) {};
		\node [style=node] (23) at (3.5, 2.5) {};
		\node [style=node] (25) at (6.5, 2.5) {};
		\node [style=none] (26) at (11.5, -1) {};
		\node [style=none] (27) at (11.5, 6) {};
		\node [style=node] (29) at (3.5, -0.5) {};
		\node [style=node] (30) at (6.5, -0.5) {};
		\node [style=none] (38) at (7, 4) {};
		\node [style=none] (39) at (11, 4) {};
		\node [style=none] (40) at (9, 5.5) {};
		\node [style=none] (42) at (7.75, 4) {};
		\node [style=none] (44) at (10.25, 4) {};
		\node [style=none] (45) at (9.5, 4) {};
		\node [style=none] (46) at (8.5, 4) {};
		\node [style=none] (47) at (12.5, 1) {$\bm{=}$};
		\node [style=none] (48) at (14.5, -2.5) {};
		\node [style=node] (49) at (18, -2) {};
		\node [style=none] (51) at (26, -2.5) {};
		\node [style=node] (52) at (15, -2) {};
		\node [style=none] (54) at (14.5, 4.5) {};
		\node [style=node] (55) at (18, 1) {};
		\node [style=node] (56) at (21, 1) {};
		\node [style=none] (57) at (26, 4.5) {};
		\node [style=node] (58) at (18, -2) {};
		\node [style=none] (60) at (21.5, 2.5) {};
		\node [style=none] (61) at (25.5, 2.5) {};
		\node [style=none] (62) at (23.5, 4) {};
		\node [style=none] (63) at (22.25, 2.5) {};
		\node [style=none] (64) at (24.75, 2.5) {};
		\node [style=none] (65) at (24, 2.5) {};
		\node [style=none] (66) at (23, 2.5) {};
		\node [style=node] (67) at (21, -2) {};
		\node [style=node] (68) at (24, -2) {};
		\node [style=none] (69) at (27, 1) {$\bm{=}$};
		\node [style=none] (70) at (30, -2.5) {};
		\node [style=node] (71) at (33.5, -2) {};
		\node [style=none] (72) at (41.5, -2.5) {};
		\node [style=node] (73) at (30.5, -2) {};
		\node [style=none] (74) at (30, 4.5) {};
		\node [style=node] (75) at (33.5, 1) {};
		\node [style=node] (76) at (36.5, 1) {};
		\node [style=none] (77) at (41.5, 4.5) {};
		\node [style=node] (78) at (33.5, -2) {};
		\node [style=none] (79) at (37, 2.5) {};
		\node [style=none] (80) at (41, 2.5) {};
		\node [style=none] (81) at (39, 4) {};
		\node [style=none] (82) at (37.75, 2.5) {};
		\node [style=none] (83) at (40.25, 2.5) {};
		\node [style=none] (84) at (39.5, 2.5) {};
		\node [style=none] (85) at (38.5, 2.5) {};
		\node [style=node] (86) at (36.5, -2) {};
		\node [style=node] (87) at (39.5, -2) {};
		\node [style=none] (88) at (28.5, 1) {\Large$\bm{-n}$};
		\node [style=none] (89) at (42.5, 1) {$\bm{=}$};
		\node [style=none] (90) at (45.5, -1) {};
		\node [style=node] (91) at (49, -0.5) {};
		\node [style=none] (92) at (55.5, -1) {};
		\node [style=node] (93) at (46, -0.5) {};
		\node [style=none] (94) at (45.5, 3) {};
		\node [style=node] (95) at (49, 2.5) {};
		\node [style=node] (96) at (52, 2.5) {};
		\node [style=none] (97) at (55.5, 3) {};
		\node [style=node] (98) at (49, -0.5) {};
		\node [style=node] (106) at (52, -0.5) {};
		\node [style=node] (107) at (55, -0.5) {};
		\node [style=none] (108) at (44, 1) {\Large$\bm{-n}$};
		\node [style=none] (111) at (13.5, 1) {\Large$\bm{n}$};
	\end{pgfonlayer}
	\begin{pgfonlayer}{edgelayer}
		\draw [style={dash_2}] (0.center) to (1.center);
		\draw [style={dash_2}] (1.center) to (16.center);
		\draw [style={dash_2}] (16.center) to (17.center);
		\draw [style=thick, bend right=60] (5) to (12);
		\draw [style=thick, bend left=60] (11) to (18);
		\draw [style={dash_2}] (21.center) to (22.center);
		\draw [style={dash_2}] (26.center) to (27.center);
		\draw [style={dash_2}] (27.center) to (21.center);
		\draw [style=thick, bend left=60] (29) to (30);
		\draw [style=jellyfish] (40.center)
			 to [bend left=315, looseness=0.75] (38.center)
			 to (39.center)
			 to [bend right=45, looseness=0.75] cycle;
		\draw [style=thick] (25) to (46.center);
		\draw [style=thick] (23) to (42.center);
		\draw [style=thick] (20) to (44.center);
		\draw [style=thick, in=255, out=15, looseness=0.50] (19) to (45.center);
		\draw [style={dash_2}] (57.center) to (54.center);
		\draw [style=jellyfish] (60.center)
			 to (61.center)
			 to [bend right=45, looseness=0.75] (62.center)
			 to [bend left=315, looseness=0.75] cycle;
		\draw [style=thick] (56) to (66.center);
		\draw [style=thick] (55) to (63.center);
		\draw [style=thick] (4) to (19);
		\draw [style=thick] (6) to (20);
		\draw [style=thick] (58) to (65.center);
		\draw [style=thick, in=270, out=15, looseness=0.50] (52) to (64.center);
		\draw [style=thick, bend left=60] (67) to (68);
		\draw [style={dash_2}] (54.center) to (48.center);
		\draw [style={dash_2}] (48.center) to (51.center);
		\draw [style={dash_2}] (51.center) to (57.center);
		\draw [style={dash_2}] (77.center) to (74.center);
		\draw [style=jellyfish] (79.center)
			 to (80.center)
			 to [bend right=45, looseness=0.75] (81.center)
			 to [bend left=315, looseness=0.75] cycle;
		\draw [style=thick] (76) to (85.center);
		\draw [style=thick] (75) to (82.center);
		\draw [style=thick, bend left=60] (86) to (87);
		\draw [style={dash_2}] (74.center) to (70.center);
		\draw [style={dash_2}] (70.center) to (72.center);
		\draw [style={dash_2}] (72.center) to (77.center);
		\draw [style=thick, bend right, looseness=0.50] (73) to (84.center);
		\draw [style=thick, bend right, looseness=0.50] (78) to (83.center);
		\draw [style={dash_2}] (97.center) to (94.center);
		\draw [style=thick, bend left=60] (106) to (107);
		\draw [style={dash_2}] (94.center) to (90.center);
		\draw [style={dash_2}] (90.center) to (92.center);
		\draw [style={dash_2}] (92.center) to (97.center);
	\end{pgfonlayer}
\end{tikzpicture}}
	\end{center}

	\textbf{Composition, Case 4}	

	Let $d_{\alpha_1}$ and $d_{\alpha_2}$ be the diagrams
	\begin{center}
		\scalebox{0.5}
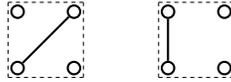
{\begin{tikzpicture}
	\begin{pgfonlayer}{nodelayer}
		\node [style=none] (0) at (3.5, 0) {};
		\node [style=none] (1) at (3.5, -4) {};
		\node [style=node] (5) at (4, -0.5) {};
		\node [style=node] (12) at (7, -0.5) {};
		\node [style=none] (16) at (7.5, -4) {};
		\node [style=none] (17) at (7.5, 0) {};
		\node [style=node] (19) at (4, -3.5) {};
		\node [style=node] (20) at (7, -3.5) {};
		\node [style=none] (21) at (11.5, 0) {};
		\node [style=none] (22) at (11.5, -4) {};
		\node [style=node] (23) at (12, -0.5) {};
		\node [style=node] (24) at (15, -0.5) {};
		\node [style=none] (25) at (15.5, -4) {};
		\node [style=none] (26) at (15.5, 0) {};
		\node [style=node] (27) at (12, -3.5) {};
		\node [style=node] (28) at (15, -3.5) {};
	\end{pgfonlayer}
	\begin{pgfonlayer}{edgelayer}
		\draw [style={dash_2}] (0.center) to (1.center);
		\draw [style={dash_2}] (1.center) to (16.center);
		\draw [style={dash_2}] (16.center) to (17.center);
		\draw [style={dash_2}] (17.center) to (0.center);
		\draw [style=thick] (19) to (12);
		\draw [style={dash_2}] (21.center) to (22.center);
		\draw [style={dash_2}] (22.center) to (25.center);
		\draw [style={dash_2}] (25.center) to (26.center);
		\draw [style={dash_2}] (26.center) to (21.center);
		\draw [style=thick] (27) to (23);
	\end{pgfonlayer}
\end{tikzpicture}}
	\end{center}
	respectively.
	Then $d_{\alpha_2} \bullet d_{\alpha_1}$ is equal to
	\begin{center}
		\scalebox{0.5}
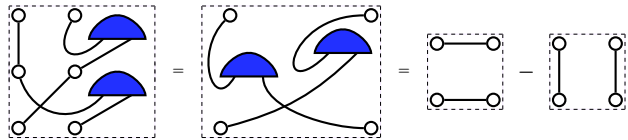
{\begin{tikzpicture}
	\begin{pgfonlayer}{nodelayer}
		\node [style=none] (21) at (4, 0) {};
		\node [style=none] (22) at (4, -4) {};
		\node [style=none] (23) at (11.75, -4) {};
		\node [style=none] (24) at (11.75, 0) {};
		\node [style=node] (25) at (7.5, -3.5) {};
		\node [style=node] (26) at (4.5, -0.5) {};
		\node [style=none] (27) at (8.25, -1.75) {};
		\node [style=none] (28) at (11.25, -1.75) {};
		\node [style=none] (29) at (9.75, -0.5) {};
		\node [style=none] (30) at (9.75, -1.75) {};
		\node [style=none] (31) at (9, -1.75) {};
		\node [style=none] (32) at (10.5, -1.75) {};
		\node [style=node] (33) at (4.5, -3.5) {};
		\node [style=none] (34) at (4, 3) {};
		\node [style=none] (35) at (4, -1) {};
		\node [style=none] (36) at (11.75, -1) {};
		\node [style=none] (37) at (11.75, 3) {};
		\node [style=node] (39) at (4.5, 2.5) {};
		\node [style=none] (40) at (8.25, 1.25) {};
		\node [style=none] (41) at (11.25, 1.25) {};
		\node [style=none] (42) at (9.75, 2.5) {};
		\node [style=none] (43) at (9.75, 1.25) {};
		\node [style=none] (44) at (9, 1.25) {};
		\node [style=none] (45) at (10.5, 1.25) {};
		\node [style=node] (47) at (7.5, -0.5) {};
		\node [style=node] (48) at (7.5, 2.5) {};
		\node [style=none] (61) at (13, -0.5) {$\bm{=}$};
		\node [style=none] (62) at (14.25, 0) {};
		\node [style=none] (63) at (14.25, -4) {};
		\node [style=none] (64) at (23.75, -4) {};
		\node [style=none] (65) at (23.75, 0) {};
		\node [style=node] (66) at (23.25, -3.5) {};
		\node [style=none] (68) at (15.25, -0.75) {};
		\node [style=none] (69) at (18.25, -0.75) {};
		\node [style=none] (70) at (16.75, 0.5) {};
		\node [style=none] (71) at (16.75, -0.75) {};
		\node [style=none] (72) at (16, -0.75) {};
		\node [style=none] (73) at (17.5, -0.75) {};
		\node [style=node] (74) at (15.25, -3.5) {};
		\node [style=none] (75) at (14.25, 3) {};
		\node [style=none] (76) at (14.25, -1) {};
		\node [style=none] (77) at (23.75, -1) {};
		\node [style=none] (78) at (23.75, 3) {};
		\node [style=node] (79) at (15.75, 2.5) {};
		\node [style=none] (80) at (20.25, 0.5) {};
		\node [style=none] (81) at (23.25, 0.5) {};
		\node [style=none] (82) at (21.75, 1.75) {};
		\node [style=none] (83) at (21.75, 0.5) {};
		\node [style=none] (84) at (21, 0.5) {};
		\node [style=none] (85) at (22.5, 0.5) {};
		\node [style=node] (87) at (23.25, 2.5) {};
		\node [style=none] (88) at (26.25, 1.5) {};
		\node [style=none] (89) at (26.25, -2.5) {};
		\node [style=node] (90) at (26.75, 1) {};
		\node [style=node] (91) at (29.75, 1) {};
		\node [style=none] (92) at (30.25, -2.5) {};
		\node [style=none] (93) at (30.25, 1.5) {};
		\node [style=node] (94) at (26.75, -2) {};
		\node [style=node] (95) at (29.75, -2) {};
		\node [style=none] (96) at (25, -0.5) {$\bm{=}$};
		\node [style=none] (97) at (31.5, -0.5) {\Large$\bm{-}$};
		\node [style=none] (98) at (32.75, 1.5) {};
		\node [style=none] (99) at (32.75, -2.5) {};
		\node [style=node] (100) at (33.25, 1) {};
		\node [style=node] (101) at (36.25, 1) {};
		\node [style=none] (102) at (36.75, -2.5) {};
		\node [style=none] (103) at (36.75, 1.5) {};
		\node [style=node] (104) at (33.25, -2) {};
		\node [style=node] (105) at (36.25, -2) {};
	\end{pgfonlayer}
	\begin{pgfonlayer}{edgelayer}
		\draw [style={dash_2}] (23.center) to (24.center);
		\draw [style={dash_2}] (21.center) to (22.center);
		\draw [style={dash_2}] (22.center) to (23.center);
		\draw [style=jellyfish] (27.center)
			 to [bend left=45, looseness=0.75] (29.center)
			 to [bend left=45, looseness=0.75] (28.center)
			 to cycle;
		\draw [style={dash_2}] (36.center) to (37.center);
		\draw [style={dash_2}] (37.center) to (34.center);
		\draw [style={dash_2}] (34.center) to (35.center);
		\draw [style=jellyfish] (40.center)
			 to [bend left=45, looseness=0.75] (42.center)
			 to [bend left=45, looseness=0.75] (41.center)
			 to cycle;
		\draw [style=thick] (33) to (47);
		\draw [style=thick] (26) to (39);
		\draw [style=thick] (25) to (32.center);
		\draw [style=thick, bend right=60, looseness=1.25] (26) to (31.center);
		\draw [style=thick] (47) to (45.center);
		\draw [style=thick, in=225, out=-120, looseness=2.75] (48) to (44.center);
		\draw [style={dash_2}] (64.center) to (65.center);
		\draw [style={dash_2}] (62.center) to (63.center);
		\draw [style={dash_2}] (63.center) to (64.center);
		\draw [style=jellyfish] (70.center)
			 to [bend left=45, looseness=0.75] (69.center)
			 to (68.center)
			 to [bend left=45, looseness=0.75] cycle;
		\draw [style={dash_2}] (77.center) to (78.center);
		\draw [style={dash_2}] (78.center) to (75.center);
		\draw [style={dash_2}] (75.center) to (76.center);
		\draw [style=jellyfish] (80.center)
			 to [bend left=45, looseness=0.75] (82.center)
			 to [bend left=45, looseness=0.75] (81.center)
			 to cycle;
		\draw [style=thick, in=180, out=-90, looseness=0.75] (73.center) to (66);
		\draw [style=thick, bend right=15] (74) to (85.center);
		\draw [style=thick, in=240, out=225, looseness=2.00] (79) to (72.center);
		\draw [style=thick, in=225, out=-180, looseness=4.00] (87) to (84.center);
		\draw [style={dash_2}] (88.center) to (89.center);
		\draw [style={dash_2}] (89.center) to (92.center);
		\draw [style={dash_2}] (92.center) to (93.center);
		\draw [style={dash_2}] (93.center) to (88.center);
		\draw [style={dash_2}] (98.center) to (99.center);
		\draw [style={dash_2}] (99.center) to (102.center);
		\draw [style={dash_2}] (102.center) to (103.center);
		\draw [style={dash_2}] (103.center) to (98.center);
		\draw [style=thick] (90) to (91);
		\draw [style=thick] (94) to (95);
		\draw [style=thick] (100) to (104);
		\draw [style=thick] (101) to (105);
	\end{pgfonlayer}
\end{tikzpicture}}
	\end{center}

	\textbf{Tensor Product, Case 4}	

	Let $d_{\alpha_1}$ and $d_{\alpha_2}$ be as above.
	Then $d_{\alpha_1} \otimes d_{\alpha_2}$ is equal to
	\begin{center}
		\scalebox{0.5}
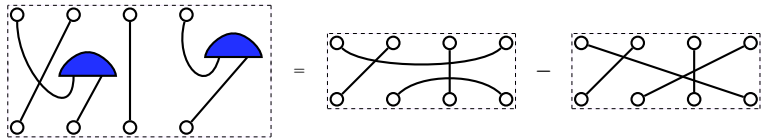
{\begin{tikzpicture}
	\begin{pgfonlayer}{nodelayer}
		\node [style=none] (0) at (0, 0) {};
		\node [style=none] (1) at (0, -4) {};
		\node [style=none] (2) at (14, -4) {};
		\node [style=none] (3) at (14, 0) {};
		\node [style=node] (4) at (9.5, -3.5) {};
		\node [style=none] (5) at (2.75, -0.75) {};
		\node [style=none] (6) at (5.75, -0.75) {};
		\node [style=none] (7) at (4.25, 0.5) {};
		\node [style=none] (8) at (4.25, -0.75) {};
		\node [style=none] (9) at (3.5, -0.75) {};
		\node [style=none] (10) at (5, -0.75) {};
		\node [style=node] (11) at (0.5, -3.5) {};
		\node [style=none] (12) at (0, 3) {};
		\node [style=none] (13) at (0, -1) {};
		\node [style=none] (14) at (14, -1) {};
		\node [style=none] (15) at (14, 3) {};
		\node [style=node] (16) at (0.5, 2.5) {};
		\node [style=none] (17) at (10.5, 0.25) {};
		\node [style=none] (18) at (13.5, 0.25) {};
		\node [style=none] (19) at (12, 1.5) {};
		\node [style=none] (20) at (12, 0.25) {};
		\node [style=none] (21) at (11.25, 0.25) {};
		\node [style=none] (22) at (12.75, 0.25) {};
		\node [style=node] (23) at (9.5, 2.5) {};
		\node [style=none] (24) at (17, 1.5) {};
		\node [style=none] (25) at (17, -2.5) {};
		\node [style=node] (26) at (17.5, 1) {};
		\node [style=node] (27) at (20.5, 1) {};
		\node [style=none] (28) at (27, -2.5) {};
		\node [style=none] (29) at (27, 1.5) {};
		\node [style=node] (30) at (17.5, -2) {};
		\node [style=node] (31) at (20.5, -2) {};
		\node [style=none] (32) at (15.5, -0.5) {$\bm{=}$};
		\node [style=none] (33) at (28.5, -0.5) {\Large$\bm{-}$};
		\node [style=node] (42) at (3.5, 2.5) {};
		\node [style=node] (43) at (6.5, 2.5) {};
		\node [style=node] (44) at (3.5, -3.5) {};
		\node [style=node] (45) at (6.5, -3.5) {};
		\node [style=node] (46) at (23.5, 1) {};
		\node [style=node] (47) at (26.5, 1) {};
		\node [style=node] (48) at (23.5, -2) {};
		\node [style=node] (49) at (26.5, -2) {};
		\node [style=none] (50) at (30, 1.5) {};
		\node [style=none] (51) at (30, -2.5) {};
		\node [style=node] (52) at (30.5, 1) {};
		\node [style=node] (53) at (33.5, 1) {};
		\node [style=none] (54) at (40, -2.5) {};
		\node [style=none] (55) at (40, 1.5) {};
		\node [style=node] (56) at (30.5, -2) {};
		\node [style=node] (57) at (33.5, -2) {};
		\node [style=node] (58) at (36.5, 1) {};
		\node [style=node] (59) at (39.5, 1) {};
		\node [style=node] (60) at (36.5, -2) {};
		\node [style=node] (61) at (39.5, -2) {};
	\end{pgfonlayer}
	\begin{pgfonlayer}{edgelayer}
		\draw [style={dash_2}] (2.center) to (3.center);
		\draw [style={dash_2}] (0.center) to (1.center);
		\draw [style={dash_2}] (1.center) to (2.center);
		\draw [style=jellyfish] (7.center)
			 to [bend left=45, looseness=0.75] (6.center)
			 to (5.center)
			 to [bend left=45, looseness=0.75] cycle;
		\draw [style={dash_2}] (14.center) to (15.center);
		\draw [style={dash_2}] (15.center) to (12.center);
		\draw [style={dash_2}] (12.center) to (13.center);
		\draw [style=jellyfish] (17.center)
			 to [bend left=45, looseness=0.75] (19.center)
			 to [bend left=45, looseness=0.75] (18.center)
			 to cycle;
		\draw [style=thick, in=270, out=270, looseness=1.75] (16) to (9.center);
		\draw [style={dash_2}] (24.center) to (25.center);
		\draw [style={dash_2}] (25.center) to (28.center);
		\draw [style={dash_2}] (28.center) to (29.center);
		\draw [style={dash_2}] (29.center) to (24.center);
		\draw [style=thick] (44) to (10.center);
		\draw [style=thick] (11) to (42);
		\draw [style=thick] (45) to (43);
		\draw [style=thick, in=255, out=-105, looseness=2.25] (23) to (21.center);
		\draw [style=thick] (4) to (22.center);
		\draw [style={dash_2}] (50.center) to (51.center);
		\draw [style={dash_2}] (51.center) to (54.center);
		\draw [style={dash_2}] (54.center) to (55.center);
		\draw [style={dash_2}] (55.center) to (50.center);
		\draw [style=thick] (48) to (46);
		\draw [style=thick] (30) to (27);
		\draw [style=thick, bend right=45, looseness=0.50] (26) to (47);
		\draw [style=thick, bend left=45, looseness=0.75] (31) to (49);
		\draw [style=thick] (52) to (61);
		\draw [style=thick] (57) to (59);
		\draw [style=thick] (60) to (58);
		\draw [style=thick] (56) to (53);
	\end{pgfonlayer}
\end{tikzpicture}}
	\end{center}

	\section{Matrix Definitions
		for 
	$\Hom_{G}((\mathbb{R}^{n})^{\otimes k}, (\mathbb{R}^{n})^{\otimes l})$
	}

	In Section \ref{groupequivlinlayers},
	we said that there exists a spanning set or a basis of
	$\Hom_{G}((\mathbb{R}^{n})^{\otimes k}, (\mathbb{R}^{n})^{\otimes l})$ for each of the groups in question, expressed in the basis of matrix units of 
$\Hom((\mathbb{R}^{n})^{\otimes k}, (\mathbb{R}^{n})^{\otimes l})$, without stating what the matrices actually are
in the main part of the paper. We state what these matrices are below.
	These definitions can also be found in \cite[Theorem 5.4]{godfrey} for $S_n$, and in \cite[Theorems 6.5, 6.6, 6.7]{pearcecrumpB} for $O(n), Sp(n)$ and $SO(n)$.

	Recall that, for any $k,l \in \mathbb{Z}_{\geq 0}$, 
as a result of picking the standard/symplectic basis for $\mathbb{R}^{n}$,
the vector space
$\Hom((\mathbb{R}^{n})^{\otimes k}, (\mathbb{R}^{n})^{\otimes l})$
has a standard basis of matrix units
\begin{equation} \label{standardbasisunits}
	\{E_{I,J}\}_{I \in [n]^l, J \in [n]^k}
\end{equation}
where $I$ is a tuple $(i_1, i_2, \dots, i_l) \in [n]^l$, 
$J$ is a tuple $(j_1, j_2, \dots, j_k) \in [n]^k$
and $E_{I,J}$ has a $1$ in the $(I,J)$ position and is $0$ elsewhere.
If one or both of $k$, $l$ is equal to $0$, then we replace the tuple that indexes the matrix by a $1$.
For example, when $k = 0$ and $l \in \mathbb{Z}_{\geq 1}$, (\ref{standardbasisunits}) becomes $\{E_{I,1}\}_{I \in [n]^l}$.

	\begin{defn}[$E_\pi$ given in Theorems \ref{diagbasisSn} and \ref{spanningsetO(n)}] \label{Sndiagdefn}

		Suppose that $d_\pi$ is a $(k,l)$--partition diagram.
		Then $E_\pi$ is defined as follows.

		Associate the indices $i_1, i_2, \dots, i_l$ with the vertices in the top row of $d_\pi$, and $j_1, j_2, \dots, j_k$ with the vertices in the bottom row of $d_\pi$.
	Then, if $S_\pi((I,J))$ is defined to be the set
	\begin{equation} \label{Snindexingset}
		\{(I,J) \in [n]^{l+k} \mid \text{if } x,y \text{ are in the same block of } \pi, \text{then } i_x = i_y \}
	\end{equation}
	(where we have momentarily replaced the elements of $J$ by $i_{l+m} \coloneqq j_m$ for all $m \in [k]$),
	 we have that
	\begin{equation} \label{mappeddiagbasisSn}
		E_\pi
		\coloneqq
		\sum_{I \in [n]^l, J \in [n]^k}
		\delta_{\pi, (I,J)}
		E_{I,J}
	\end{equation}
	where
	\begin{equation}
		\delta_{\pi, (I,J)}
		\coloneqq
		\begin{cases}
			1 & \text{if } (I,J) \in S_\pi((I,J)) \\
			0 & \text{otherwise}
		\end{cases}
	\end{equation}
	\end{defn}

	We give an example in Figure \ref{matrix1,4}.

\begin{figure}[tb]
	\begin{center}
\begin{tblr}{
  colspec = {X[c,h]X[c]X[c]},
  stretch = 0,
  rowsep = 8pt,
  hlines = {1pt},
  vlines = {1pt},
}
	{Partition Diagram \\ $d_\pi$} 	& {$S_\pi((I,J))$}
	& 
	{Matrix Element \\ with $n = 4$} 	\\
	\scalebox{0.6}{\begin{tikzpicture}
	\begin{pgfonlayer}{nodelayer}
		\node [style=none] (0) at (1, 1) {$1$};
		\node [style=node] (1) at (1, 0) {};
		\node [style=node] (2) at (1, -3) {};
		\node [style=none] (5) at (1, -4) {$2$};
		\node [style=none] (8) at (0.5, 0.5) {};
		\node [style=none] (9) at (0.5, -3.5) {};
		\node [style=none] (10) at (1.5, -3.5) {};
		\node [style=none] (11) at (1.5, 0.5) {};
	\end{pgfonlayer}
	\begin{pgfonlayer}{edgelayer}
		\draw [style={dash_2}] (8.center)
			 to (9.center)
			 to (10.center)
			 to (11.center)
			 to cycle;
		\draw [style=thick] (2) to (1);
	\end{pgfonlayer}
\end{tikzpicture}} & 
				$\{(i,i) \in [n]^2\}$
				&
	\scalebox{0.75}{
	$
	\NiceMatrixOptions{code-for-first-row = \scriptstyle \color{blue},
                   	   code-for-first-col = \scriptstyle \color{blue}
	}
	\begin{bNiceArray}{*{4}{c}}[first-row,first-col]
				& 1 	& 2	& 3	& 4 	\\
		1		& 1	& 0	& 0	& 0	\\
		2		& 0	& 1	& 0	& 0	\\
		3		& 0	& 0	& 1	& 0	\\
		4		& 0	& 0	& 0	& 1
	\end{bNiceArray}
	$}
	\\
	\scalebox{0.6}{\begin{tikzpicture}
	\begin{pgfonlayer}{nodelayer}
		\node [style=none] (0) at (1, 1) {$1$};
		\node [style=node] (1) at (1, 0) {};
		\node [style=node] (2) at (1, -3) {};
		\node [style=none] (3) at (1, -4) {$2$};
		\node [style=none] (4) at (0.5, 0.5) {};
		\node [style=none] (5) at (0.5, -3.5) {};
		\node [style=none] (6) at (1.5, -3.5) {};
		\node [style=none] (7) at (1.5, 0.5) {};
	\end{pgfonlayer}
	\begin{pgfonlayer}{edgelayer}
		\draw [style={dash_2}] (5.center)
			 to (6.center)
			 to (7.center)
			 to (4.center)
			 to cycle;
	\end{pgfonlayer}
\end{tikzpicture}}	
	&
				$\{(i,j) \in [n]^2\}$
	& 
	\scalebox{0.75}{
	$
	\NiceMatrixOptions{code-for-first-row = \scriptstyle \color{blue},
                   	   code-for-first-col = \scriptstyle \color{blue}
	}
	\begin{bNiceArray}{*{4}{c}}[first-row,first-col]
				& 1 	& 2	& 3	& 4 	\\
		1		& 1	& 1	& 1	& 1	\\
		2		& 1	& 1	& 1	& 1	\\
		3		& 1	& 1	& 1	& 1	\\
		4		& 1	& 1	& 1	& 1
	\end{bNiceArray}
	$}
	\\
\end{tblr}
		\caption{A table showing how $(1,1)$--partition diagrams correspond to matrices in 
		$\Hom_{S_4}(\mathbb{R}^{4}, \mathbb{R}^{4})$.}
  	\label{matrix1,4}
	\end{center}
\end{figure}

	\begin{defn}[$F_\beta$ given in Theorem \ref{spanningsetSp(n)}]
		Suppose that $d_\beta$ is a $(k,l)$--Brauer diagram.
		Then $F_\beta$ is defined as follows.
	
		Associate the indices $i_1, i_2, \dots, i_l$ with the vertices in the top row of $d_\beta$, and $j_1, j_2, \dots, j_k$ with the vertices in the bottom row of $d_\beta$.
	Then, we have that
	\begin{equation} \label{matrixSp(n)}
		F_\beta 
		\coloneqq
		\sum_{I, J} 
		\gamma_{r_1, u_1}
		\gamma_{r_2, u_2}
		\dots
		\gamma_{r_{\frac{l+k}{2}}, u_{\frac{l+k}{2}}}
		E_{I,J}
	\end{equation}
	where the indices $i_p, j_p$ range over $1, 1', \dots, m, m'$,
	where $r_1, u_1, \dots, r_{\frac{l+k}{2}}, u_{\frac{l+k}{2}}$ is any permutation of the indices $i_1, i_2, \dots, i_l, j_1, j_2, \dots, j_k$ such that the vertices corresponding to
	$r_p, u_p$ 
	are in the same block of $\beta$, and
	\begin{equation} \label{gammarpup}
		\gamma_{r_p, u_p} \coloneqq
		\begin{cases}
			\delta_{r_p, u_p} & \text{if the vertices corresponding to } r_p, u_p \text{ are in different rows of } d_\beta \\
			\epsilon_{r_p, u_p} & \text{if the vertices corresponding to } r_p, u_p \text{ are in the same row of } d_\beta
    		\end{cases}
	\end{equation}
	Here, $\epsilon_{r_p, u_p}$ is given by
\begin{equation} \label{epsilondef1}
	\epsilon_{\alpha, \beta} = \epsilon_{{\alpha'}, {\beta'}} = 0
\end{equation}
\begin{equation} \label{epsilondef2}
	\epsilon_{\alpha, {\beta'}} = - \epsilon_{{\alpha'}, {\beta}} = \delta_{\alpha, \beta}
\end{equation}
	\end{defn}

	We give an example in Figure \ref{matrix2,2}.

	\begin{defn}[$H_\alpha$ given in Theorem \ref{spanningsetSO(n)}]
		Suppose that $d_\alpha$ is a $(l+k)\backslash n$--diagram.
		Then $H_\alpha$ is defined as follows.

	Associate 
the indices $i_1, i_2, \dots, i_l$ with the vertices in the top row of $d_\alpha$, and $j_1, j_2, \dots, j_k$ with the vertices in the bottom row of $d_\alpha$.
	Suppose that there are $s$ free vertices in the top row. Then there are $n-s$ free vertices in the bottom row.
	Relabel the $s$ free indices in the top row (from left to right) by 
	$t_1, \dots, t_s$, and the $n-s$ free indices in the bottom row (from left to right) by $b_1, \dots, b_{n-s}$. 

	Then, define
		$
		\chi
			\left(\begin{smallmatrix} 
				1 & 2 & \cdots & s & s+1 & \cdots & n\\
				t_1 & t_2 & \cdots & t_s & b_1 & \cdots & b_{n-s}
			\end{smallmatrix}\right)
		$
		as follows: it is 
		$0$ if the elements $t_1, \dots, t_s, b_1, \dots, b_{n-s}$ are not distinct, otherwise, it is
		$	
		\sgn
			\left(\begin{smallmatrix} 
				1 & 2 & \cdots & s & s+1 & \cdots & n\\
				t_1 & t_2 & \cdots & t_s & b_1 & \cdots & b_{n-s}
			\end{smallmatrix}\right)
		$,
		considered as a permutation of $[n]$.

	As a result, for any $n \in \mathbb{Z}_{\geq 1}$, we have that
	\begin{equation} \label{SO(n)Halpha}
		H_\alpha
		\coloneqq
		\sum_{I \in [n]^l, J \in [n]^k} 
		\chi
			\left(\begin{smallmatrix} 
				1 & 2 & \cdots & s & s+1 & \cdots & n\\
				t_1 & t_2 & \cdots & t_s & b_1 & \cdots & b_{n-s}
			\end{smallmatrix}\right)
		\delta_{r_1, u_1}
		\delta_{r_2, u_2}
		\dots
		\delta_{r_{\frac{l+k-n}{2}}, u_{\frac{l+k-n}{2}}}
		E_{I,J}
	\end{equation}
	Here, $r_1, u_1, \dots, r_{\frac{l+k-n}{2}}, u_{\frac{l+k-n}{2}}$
	is any permutation of the indices 
	\begin{equation}	
	\{i_1, \dots, i_l, j_1, \dots, j_k\} \backslash \{t_1, \dots, t_s, b_1, \dots, b_{n-s}\}
	\end{equation}
	such that the vertices corresponding to $r_p, u_p$ are in the same block of $\alpha$. 
	\end{defn}

	We give an example in Figure \ref{so3matrix2,2}.

\begin{figure}[t]
	\begin{center}
\begin{tblr}{
  colspec = {X[c,h]X[c]X[c]},
  stretch = 0,
  rowsep = 6pt,
  hlines = {1pt},
  vlines = {1pt},
}
	{Brauer Diagram \\ $d_\beta$} 	& {Matrix Entries}	& 
	{Matrix Element \\ with $n = 2$}\\
	\scalebox{0.6}{\begin{tikzpicture}
	\begin{pgfonlayer}{nodelayer}
		\node [style=none] (0) at (0, 1) {$1$};
		\node [style=node] (1) at (0, 0) {};
		\node [style=node] (2) at (0, -3) {};
		\node [style=node] (3) at (3, 0) {};
		\node [style=node] (4) at (3, -3) {};
		\node [style=none] (5) at (3, 1) {$2$};
		\node [style=none] (6) at (0, -4) {$3$};
		\node [style=none] (7) at (3, -4) {$4$};
		\node [style=none] (8) at (-0.5, 0.5) {};
		\node [style=none] (9) at (-0.5, -3.5) {};
		\node [style=none] (10) at (3.5, -3.5) {};
		\node [style=none] (11) at (3.5, 0.5) {};
	\end{pgfonlayer}
	\begin{pgfonlayer}{edgelayer}
		\draw [style={dash_2}] (8.center)
			 to (9.center)
			 to (10.center)
			 to (11.center)
			 to cycle;
		\draw [style=thick] (2) to (4);
		\draw [style=thick] (3) to (1);
	\end{pgfonlayer}
\end{tikzpicture}} & $(\epsilon_{i_1, i_2}\epsilon_{j_1,j_2})$
	& 
	\scalebox{0.75}{
	$
	\NiceMatrixOptions{code-for-first-row = \scriptstyle \color{blue},
                   	   code-for-first-col = \scriptstyle \color{blue}
	}
	\begin{bNiceArray}{*{2}{c}*{2}{c}}[first-row,first-col]
				& 1,1 	& 1,1'	& 1,1'	& 1',1'	\\
		1,1		& 0	& 0	& 0	& 0	\\
		1,1'		& 0	& 1	& -1	& 0	\\
		1',1		& 0	& -1	& 1	& 0	\\
		1',1'		& 0	& 0	& 0	& 0
	\end{bNiceArray}
	$}
	\\
	\scalebox{0.6}{\begin{tikzpicture}
	\begin{pgfonlayer}{nodelayer}
		\node [style=none] (0) at (0, 1) {$1$};
		\node [style=node] (1) at (0, 0) {};
		\node [style=node] (2) at (0, -3) {};
		\node [style=node] (3) at (3, 0) {};
		\node [style=node] (4) at (3, -3) {};
		\node [style=none] (5) at (3, 1) {$2$};
		\node [style=none] (6) at (0, -4) {$3$};
		\node [style=none] (7) at (3, -4) {$4$};
		\node [style=none] (8) at (-0.5, 0.5) {};
		\node [style=none] (9) at (-0.5, -3.5) {};
		\node [style=none] (10) at (3.5, -3.5) {};
		\node [style=none] (11) at (3.5, 0.5) {};
	\end{pgfonlayer}
	\begin{pgfonlayer}{edgelayer}
		\draw [style={dash_2}] (8.center)
			 to (9.center)
			 to (10.center)
			 to (11.center)
			 to cycle;
		\draw [style=thick] (2) to (1);
		\draw [style=thick] (4) to (3);
	\end{pgfonlayer}
\end{tikzpicture}}	& 
	$(\delta_{i_1, j_1}\delta_{i_2,j_2})$
	& 
	\scalebox{0.75}{
	$
	\NiceMatrixOptions{code-for-first-row = \scriptstyle \color{blue},
                   	   code-for-first-col = \scriptstyle \color{blue}
	}
	\begin{bNiceArray}{*{2}{c}*{2}{c}}[first-row,first-col]
				& 1,1 	& 1,1'	& 1,1'	& 1',1'	\\
		1,1		& 1	& 0	& 0	& 0	\\
		1,1'		& 0	& 1	& 0	& 0	\\
		1',1		& 0	& 0	& 1	& 0	\\
		1',1'		& 0	& 0	& 0	& 1
	\end{bNiceArray}
	$}
	\\
	\scalebox{0.6}{\begin{tikzpicture}
	\begin{pgfonlayer}{nodelayer}
		\node [style=none] (0) at (0, 1) {$1$};
		\node [style=node] (1) at (0, 0) {};
		\node [style=node] (2) at (0, -3) {};
		\node [style=node] (3) at (3, 0) {};
		\node [style=node] (4) at (3, -3) {};
		\node [style=none] (5) at (3, 1) {$2$};
		\node [style=none] (6) at (0, -4) {$3$};
		\node [style=none] (7) at (3, -4) {$4$};
		\node [style=none] (8) at (-0.5, 0.5) {};
		\node [style=none] (9) at (-0.5, -3.5) {};
		\node [style=none] (10) at (3.5, -3.5) {};
		\node [style=none] (11) at (3.5, 0.5) {};
	\end{pgfonlayer}
	\begin{pgfonlayer}{edgelayer}
		\draw [style={dash_2}] (8.center)
			 to (9.center)
			 to (10.center)
			 to (11.center)
			 to cycle;
		\draw [style=thick] (1) to (4);
		\draw [style=thick] (2) to (3);
	\end{pgfonlayer}
\end{tikzpicture}}	& 
	$(\delta_{i_1, j_2}\delta_{i_2,j_1})$
	& 
	\scalebox{0.75}{
	$
	\NiceMatrixOptions{code-for-first-row = \scriptstyle \color{blue},
                   	   code-for-first-col = \scriptstyle \color{blue}
	}
	\begin{bNiceArray}{*{2}{c}*{2}{c}}[first-row,first-col]
				& 1,1 	& 1,1'	& 1,1'	& 1',1'	\\
		1,1		& 1	& 0	& 0	& 0	\\
		1,1'		& 0	& 0	& 1	& 0	\\
		1',1		& 0	& 1	& 0	& 0	\\
		1',1'		& 0	& 0	& 0	& 1
	\end{bNiceArray}
	$}
	\\
\end{tblr}
		\caption{A table showing how $(2,2)$--Brauer diagrams correspond to matrices in 
		$\Hom_{Sp(2)}((\mathbb{R}^{2})^{\otimes 2}, (\mathbb{R}^{2})^{\otimes 2})$.}
  	\label{matrix2,2}
	\end{center}
\end{figure}

\begin{figure}[ht]
	\begin{center}
\begin{tblr}{
  colspec = {X[c,h]X[c]X[c]},
  stretch = 0,
  rowsep = 6pt,
  hlines = {1pt},
  vlines = {1pt},
}
	{Diagram \\ $d_\alpha$} 	& {Matrix Entries}	& 
	{Matrix Element \\ with $n = 2$}\\
	\scalebox{0.6}{\begin{tikzpicture}
	\begin{pgfonlayer}{nodelayer}
		\node [style=none] (0) at (5, 0) {$1$};
		\node [style=none] (1) at (8, 0) {$2$};
		\node [style=none] (2) at (5, -5) {$3$};
		\node [style=none] (3) at (8, -5) {$4$};
		\node [style=node] (10) at (5, -1) {};
		\node [style=node] (14) at (8, -1) {};
		\node [style=node] (15) at (5, -4) {};
		\node [style=node] (16) at (8, -4) {};
		\node [style=none] (20) at (4.5, -0.5) {};
		\node [style=none] (21) at (4.5, -4.5) {};
		\node [style=none] (22) at (8.5, -4.5) {};
		\node [style=none] (23) at (8.5, -0.5) {};
	\end{pgfonlayer}
	\begin{pgfonlayer}{edgelayer}
		\draw [style={dash_3}] (21.center)
			 to (22.center)
			 to (23.center)
			 to (20.center)
			 to cycle;
		\draw [style=thick] (14) to (10);
	\end{pgfonlayer}
\end{tikzpicture}} & 
	$(
	\chi
			\left(\begin{smallmatrix} 
				1 & 2 \\
				j_1 & j_2 
			\end{smallmatrix}\right)
	\delta_{i_1,i_2})$
	& 
	\scalebox{0.75}{
	$
	\NiceMatrixOptions{code-for-first-row = \scriptstyle \color{blue},
                   	   code-for-first-col = \scriptstyle \color{blue}
	}
	\begin{bNiceArray}{*{2}{c}*{2}{c}}[first-row,first-col]
				& 1,1 	& 1,2	& 2,1	& 2,2	\\
		1,1		& 0	& 1	& -1	& 0	\\
		1,2		& 0	& 0	& 0	& 0	\\
		2,1		& 0	& 0	& 0	& 0	\\
		2,2		& 0	& 1	& -1	& 0
	\end{bNiceArray}
	$}
	\\
	\scalebox{0.6}{\begin{tikzpicture}
	\begin{pgfonlayer}{nodelayer}
		\node [style=none] (0) at (5, 0) {$1$};
		\node [style=none] (1) at (8, 0) {$2$};
		\node [style=none] (2) at (5, -5) {$3$};
		\node [style=none] (3) at (8, -5) {$4$};
		\node [style=node] (4) at (5, -1) {};
		\node [style=node] (5) at (8, -1) {};
		\node [style=node] (6) at (5, -4) {};
		\node [style=node] (7) at (8, -4) {};
		\node [style=none] (8) at (4.5, -0.5) {};
		\node [style=none] (9) at (4.5, -4.5) {};
		\node [style=none] (10) at (8.5, -4.5) {};
		\node [style=none] (11) at (8.5, -0.5) {};
	\end{pgfonlayer}
	\begin{pgfonlayer}{edgelayer}
		\draw [style={dash_3}] (8.center)
			 to (9.center)
			 to (10.center)
			 to (11.center)
			 to cycle;
		\draw [style=thick] (4) to (6);
	\end{pgfonlayer}
\end{tikzpicture}}	& 
	$(
	\chi
			\left(\begin{smallmatrix} 
				1 & 2 \\
				i_2 & j_2 
			\end{smallmatrix}\right)
	\delta_{i_1,j_1})$
	& 
	\scalebox{0.75}{
	$
	\NiceMatrixOptions{code-for-first-row = \scriptstyle \color{blue},
                   	   code-for-first-col = \scriptstyle \color{blue}
	}
	\begin{bNiceArray}{*{2}{c}*{2}{c}}[first-row,first-col]
				& 1,1 	& 1,2	& 2,1	& 2,2	\\
		1,1		& 0	& 1	& 0	& 0	\\
		1,2		& -1	& 0	& 0	& 0	\\
		2,1		& 0	& 0	& 0	& 1	\\
		2,2		& 0	& 0	& -1	& 0
	\end{bNiceArray}
	$}
	\\
	\scalebox{0.6}{\begin{tikzpicture}
	\begin{pgfonlayer}{nodelayer}
		\node [style=none] (0) at (5, 0) {$1$};
		\node [style=none] (1) at (8, 0) {$2$};
		\node [style=none] (2) at (5, -5) {$3$};
		\node [style=none] (3) at (8, -5) {$4$};
		\node [style=node] (4) at (5, -1) {};
		\node [style=node] (5) at (8, -1) {};
		\node [style=node] (6) at (5, -4) {};
		\node [style=node] (7) at (8, -4) {};
		\node [style=none] (8) at (4.5, -0.5) {};
		\node [style=none] (9) at (4.5, -4.5) {};
		\node [style=none] (10) at (8.5, -4.5) {};
		\node [style=none] (11) at (8.5, -0.5) {};
	\end{pgfonlayer}
	\begin{pgfonlayer}{edgelayer}
		\draw [style={dash_3}] (8.center)
			 to (9.center)
			 to (10.center)
			 to (11.center)
			 to cycle;
		\draw [style=thick] (4) to (7);
	\end{pgfonlayer}
\end{tikzpicture}}	& 
	$(
	\chi
			\left(\begin{smallmatrix} 
				1 & 2 \\
				i_2 & j_1 
			\end{smallmatrix}\right)
	\delta_{i_1,j_2})$
	& 
	\scalebox{0.75}{
	$
	\NiceMatrixOptions{code-for-first-row = \scriptstyle \color{blue},
                   	   code-for-first-col = \scriptstyle \color{blue}
	}
	\begin{bNiceArray}{*{2}{c}*{2}{c}}[first-row,first-col]
				& 1,1 	& 1,2	& 2,1	& 2,2	\\
		1,1		& 0	& 0	& 1	& 0	\\
		1,2		& -1	& 0	& 0	& 0	\\
		2,1		& 0	& 0	& 0	& 1	\\
		2,2		& 0	& -1	& 0	& 0
	\end{bNiceArray}
	$}
	\\
	\scalebox{0.6}{\begin{tikzpicture}
	\begin{pgfonlayer}{nodelayer}
		\node [style=none] (0) at (5, 0) {$1$};
		\node [style=none] (1) at (8, 0) {$2$};
		\node [style=none] (2) at (5, -5) {$3$};
		\node [style=none] (3) at (8, -5) {$4$};
		\node [style=node] (4) at (5, -1) {};
		\node [style=node] (5) at (8, -1) {};
		\node [style=node] (6) at (5, -4) {};
		\node [style=node] (7) at (8, -4) {};
		\node [style=none] (8) at (4.5, -0.5) {};
		\node [style=none] (9) at (4.5, -4.5) {};
		\node [style=none] (10) at (8.5, -4.5) {};
		\node [style=none] (11) at (8.5, -0.5) {};
	\end{pgfonlayer}
	\begin{pgfonlayer}{edgelayer}
		\draw [style={dash_3}] (9.center)
			 to (10.center)
			 to (11.center)
			 to (8.center)
			 to cycle;
		\draw [style=thick] (5) to (6);
	\end{pgfonlayer}
\end{tikzpicture}}	& 
	$(
	\chi
			\left(\begin{smallmatrix} 
				1 & 2 \\
				i_1 & j_2 
			\end{smallmatrix}\right)
	\delta_{i_2,j_1})$
	& 
	\scalebox{0.75}{
	$
	\NiceMatrixOptions{code-for-first-row = \scriptstyle \color{blue},
                   	   code-for-first-col = \scriptstyle \color{blue}
	}
	\begin{bNiceArray}{*{2}{c}*{2}{c}}[first-row,first-col]
				& 1,1 	& 1,2	& 2,1	& 2,2	\\
		1,1		& 0	& 1	& 0	& 0	\\
		1,2		& 0	& 0	& 0	& 1	\\
		2,1		& -1	& 0	& 0	& 0	\\
		2,2		& 0	& 0	& -1	& 0
	\end{bNiceArray}
	$}
	\\
	\scalebox{0.6}{\begin{tikzpicture}
	\begin{pgfonlayer}{nodelayer}
		\node [style=none] (0) at (5, 0) {$1$};
		\node [style=none] (1) at (8, 0) {$2$};
		\node [style=none] (2) at (5, -5) {$3$};
		\node [style=none] (3) at (8, -5) {$4$};
		\node [style=node] (4) at (5, -1) {};
		\node [style=node] (5) at (8, -1) {};
		\node [style=node] (6) at (5, -4) {};
		\node [style=node] (7) at (8, -4) {};
		\node [style=none] (8) at (4.5, -0.5) {};
		\node [style=none] (9) at (4.5, -4.5) {};
		\node [style=none] (10) at (8.5, -4.5) {};
		\node [style=none] (11) at (8.5, -0.5) {};
	\end{pgfonlayer}
	\begin{pgfonlayer}{edgelayer}
		\draw [style={dash_3}] (9.center)
			 to (10.center)
			 to (11.center)
			 to (8.center)
			 to cycle;
		\draw [style=thick] (7) to (5);
	\end{pgfonlayer}
\end{tikzpicture}}	& 
	$(
	\chi
			\left(\begin{smallmatrix} 
				1 & 2 \\
				i_1 & j_1 
			\end{smallmatrix}\right)
	\delta_{i_2,j_2})$
	& 
	\scalebox{0.75}{
	$
	\NiceMatrixOptions{code-for-first-row = \scriptstyle \color{blue},
                   	   code-for-first-col = \scriptstyle \color{blue}
	}
	\begin{bNiceArray}{*{2}{c}*{2}{c}}[first-row,first-col]
				& 1,1 	& 1,2	& 2,1	& 2,2	\\
		1,1		& 0	& 0	& 1	& 0	\\
		1,2		& 0	& 0	& 0	& 1	\\
		2,1		& -1	& 0	& 0	& 0	\\
		2,2		& 0	& -1	& 0	& 0
	\end{bNiceArray}
	$}
	\\
	\scalebox{0.6}{\begin{tikzpicture}
	\begin{pgfonlayer}{nodelayer}
		\node [style=none] (0) at (5, 0) {$1$};
		\node [style=none] (1) at (8, 0) {$2$};
		\node [style=none] (2) at (5, -5) {$3$};
		\node [style=none] (3) at (8, -5) {$4$};
		\node [style=node] (4) at (5, -1) {};
		\node [style=node] (5) at (8, -1) {};
		\node [style=node] (6) at (5, -4) {};
		\node [style=node] (7) at (8, -4) {};
		\node [style=none] (8) at (4.5, -0.5) {};
		\node [style=none] (9) at (4.5, -4.5) {};
		\node [style=none] (10) at (8.5, -4.5) {};
		\node [style=none] (11) at (8.5, -0.5) {};
	\end{pgfonlayer}
	\begin{pgfonlayer}{edgelayer}
		\draw [style={dash_3}] (9.center)
			 to (10.center)
			 to (11.center)
			 to (8.center)
			 to cycle;
		\draw [style=thick] (6) to (7);
	\end{pgfonlayer}
\end{tikzpicture}}	& 
	$(
	\chi
			\left(\begin{smallmatrix} 
				1 & 2 \\
				i_1 & i_2 
			\end{smallmatrix}\right)
	\delta_{j_1,j_2})$
	& 
	\scalebox{0.75}{
	$
	\NiceMatrixOptions{code-for-first-row = \scriptstyle \color{blue},
                   	   code-for-first-col = \scriptstyle \color{blue}
	}
	\begin{bNiceArray}{*{2}{c}*{2}{c}}[first-row,first-col]
				& 1,1 	& 1,2	& 2,1	& 2,2	\\
		1,1		& 0	& 0	& 0	& 0	\\
		1,2		& 1	& 0	& 0	& 1	\\
		2,1		& -1	& 0	& 0	& -1	\\
		2,2		& 0	& 0	& 0	& 0
	\end{bNiceArray}
	$}
	\\
\end{tblr}
		\caption{A table showing how $(2+2) \backslash 2$--diagrams correspond to matrices in 
		$\Hom_{SO(2)}((\mathbb{R}^{2})^{\otimes 2}, (\mathbb{R}^{2})^{\otimes 2})$.}
	\label{so3matrix2,2}
	\end{center}
\end{figure}

	\section{Basics of Category Theory}

	We provide some of the basics on Category Theory to help the reader to understand some of the language that is used in the main text.
Other good references are \cite{maclane, kock, turaev}.

	\subsection{Categories}

	\begin{defn}
		A category $\mathcal{C}$
		consists of 
		\begin{itemize}
			\item a collection of \textit{objects} $Ob(\mathcal{C})$,
			\item for every $X, Y \in Ob(\mathcal{C})$, a collection $\Hom_{\mathcal{C}}(X,Y)$ of \textit{morphisms} from $X$ to $Y$
			\item for every $X, Y, Z \in Ob(\mathcal{C})$, an associative composition rule
				\begin{equation}
					\circ :
					\Hom_{\mathcal{C}}(Y,Z)
					\times
					\Hom_{\mathcal{C}}(X,Y)
					\rightarrow
					\Hom_{\mathcal{C}}(X,Z)
				\end{equation}
			\item and, for every $X \in Ob(\mathcal{C})$, there is an identity morphism $1_X \in \Hom_{\mathcal{C}}(X,X)$, such that every morphism $f \in \Hom_{\mathcal{C}}(Y,X)$
				satisfies $1_X \circ f = f$, and every morphism $g \in \Hom_{\mathcal{C}}(X,Z)$
				satisfies $g \circ 1_X = g$.
		\end{itemize}
	\end{defn}

	\begin{remark}
		We write $f: X \rightarrow Y$ for a morphism in $\Hom_{\mathcal{C}}(X,Y)$.
	\end{remark}

	\begin{defn}
		A category is said to be \textit{locally small} if the collection of morphisms between any two objects in the category is a set.
	\end{defn}

	\begin{example}
		We list some examples of categories below.
		\begin{itemize}
			\item The category \textbf{Set}, whose objects are all sets, and whose morphisms are the set mappings.
			\item The category \textbf{Vect$_{\mathbb{R}}$}, whose objects are all vector spaces over $\mathbb{R}$, and whose morphisms are the $\mathbb{R}$--linear maps.
			\item The category \textbf{Grp}, whose objects are all groups, and whose morphisms are the group homomorphisms.
		\end{itemize}
	\end{example}

	\subsection{Functors}
	
	Just as there is a notion of a map between sets, there is also a notion of a "map" between categories, called functors.

	\begin{defn}
		Suppose that $\mathcal{C}$ and $\mathcal{D}$ are two categories.

		Then a functor $\mathcal{F}: \mathcal{C} \rightarrow \mathcal{D}$ consists of
		\begin{itemize}
			\item a map $\mathcal{F}: Ob(\mathcal{C}) \rightarrow Ob(\mathcal{D})$, and
			\item for each pair of objects $X, Y \in \mathcal{C}$, a map
				\begin{equation}
					\mathcal{F}_{X,Y}: 
					\Hom_{\mathcal{C}}(X,Y)
					\rightarrow
					\Hom_{\mathcal{D}}(\mathcal{F}(X),\mathcal{F}(Y))
				\end{equation}
				that preserves 1) the composition rule in each category, that is, for morphisms $f: X \rightarrow Y$, $g: Y \rightarrow Z$ in $\mathcal{C}$, we have that
				\begin{equation}
					\mathcal{F}(g \circ_{\mathcal{C}} f)
					=
					\mathcal{F}(g) \circ_{\mathcal{D}} \mathcal{F}(f)
				\end{equation}
				and 2) the identity morphisms, that is, for all $X \in Ob(\mathcal{C})$, we have that
				\begin{equation}
					\mathcal{F}_{X,X}(1_X) = 1_{\mathcal{F}(X)} 
				\end{equation}
		\end{itemize}
	\end{defn}

	\begin{example}
		We list some examples of functors below.
		\begin{itemize}
			\item the functor \textbf{Grp} $\rightarrow$ \textbf{Set}, which to each group $G$ associates the underlying set, and to each group homomorphism, the underlying set map.
			\item the functor \textbf{Set}$ \rightarrow$ \textbf{Vect$_{\mathbb{R}}$}, which to each set associates the vector space spanned by the elements of the set, and, for a given set map, the associated linear map is the one that is given by extending linearly.
		\end{itemize}
	\end{example}

	\section{Monoidal Categories and Monoidal Functor Proofs}

	We first show that $\mathcal{C}(G)$, for the groups $G = S_n, O(n), Sp(n)$ and $SO(n)$,
	is a subcategory of $\Rep(G)$.
	\begin{proof}
		The category of representations of a group $G$, $\Rep(G)$, is defined as follows.
		\begin{defn}
			Let $G$ be a group. Then $\Rep(G)$ is the category
			whose objects are pairs $(V, \rho_V)$,
			where $\rho_V: G \rightarrow GL(V)$ is a representation of $G$,
	and,
	for any pair of objects 
	$(V, \rho_V)$ and $(W, \rho_W)$,
	the morphism space, 
			$\Hom_{Rep(G)}((V, \rho_V), (W, \rho_W))$,
	is precisely the vector space of $G$--equivariant linear maps $V \rightarrow W$,
			$\Hom_{G}(V, W)$.

	The vertical composition of morphisms is given by the usual composition of linear maps, the 
	horizontal composition of morphisms 
	is given by the usual tensor product of linear maps, both of which are associative operations, and the unit object is given by $(\mathbb{R}, 1_\mathbb{R})$, where $1_\mathbb{R}$ is the one-dimensional trivial representation of $G$.
		\end{defn}
		With this definition, it is clear that $\mathcal{C}(G)$ is a subcategory of $\Rep(G)$.
		In fact, it is a full subcategory, as, for any pair of objects in $C(G)$, 
	$((\mathbb{R}^n)^{\otimes k},\rho_k)$ and $((\mathbb{R}^n)^{\otimes l},\rho_l)$,
	the morphism space, 
	$\Hom_{\mathcal{C}(G)}(((\mathbb{R}^n)^{\otimes k},\rho_k), ((\mathbb{R}^n)^{\otimes l},\rho_l))$ 
	is precisely
	$\Hom_{\Rep(G)}(((\mathbb{R}^n)^{\otimes k},\rho_k), ((\mathbb{R}^n)^{\otimes l},\rho_l))$.
	\end{proof}

	We now prove the existence of the full, strict $\mathbb{R}$--linear monoidal functors given in Section \ref{fullstrictmonfunct}.

	\begin{proof}[Proof of Theorem \ref{partfunctor}]

		The functor $\Theta$ is full because the map (\ref{partmorphism}), for all objects $k, l$ in $\mathcal{P}(n)$, which, by definition, is (\ref{partmorphismmap}), is surjective, by Theorem \ref{diagbasisSn}.

		To show that $\Theta$ is strict, $\mathbb{R}$--linear monoidal,
		we need to show that $\Theta$ satisfies the conditions of Definition \ref{monoidalfunctordefn}.
		The picture to have in mind for point 2. below is that the tensor product on set partition diagrams places the diagrams side-by-side, without any overlapping of vertices.
		\begin{enumerate}
			\item Let $k, l$ be any two objects in $\mathcal{P}(n)$. Then 
	\begin{align}
		\Theta(k \otimes l) 
		& = \Theta(k + l) \\
		& = ((\mathbb{R}^{n})^{\otimes k+l}, \rho_{k+l}) \\
		& = ((\mathbb{R}^{n})^{\otimes k} \otimes (\mathbb{R}^{n})^{\otimes l}, \rho_k \otimes \rho_l) \\
		& = \Theta(k) \otimes \Theta(l)
	\end{align}
			\item It is enough to show this on arbitrary basis elements of arbitrary morphism spaces as the morphism spaces are vector spaces.
			Suppose that $f: k \rightarrow l$ and $g: q \rightarrow m$ are two basis elements in $\Hom_{\mathcal{P}(n)}(k,l)$ and $\Hom_{\mathcal{P}(n)}(q,m)$ respectively.
			Then $f = d_\pi$ for some set partition $\pi$ in $\Pi_{l+k}$, and $g = d_\tau$ for some set partition $\tau$ in $\Pi_{m+q}$. 

				As $\Hom_{\mathcal{P}(n)}(k,l) = P_k^l(n)$
				and $\Hom_{\mathcal{P}(n)}(q,m) = P_q^m(n)$, we have, by (\ref{tensorprod}), that $f \otimes g$ is an element of $P_{k+q}^{l+m}(n) = \Hom_{\mathcal{P}(n)}(k+q,l+m)$.
				
				In particular, $f \otimes g = d_{\omega}$ for the set partition $\omega \coloneqq \pi \cup \tau \in \Pi_{l+m+k+q}$. 
				
				By Theorem \ref{diagbasisSn}, we have that $\Theta(f) = E_\pi$, $\Theta(g) = E_\tau$, and $\Theta(f \otimes g) = E_\omega$.
	
				Hence,
	\begin{align} 
		\Theta(f) \otimes \Theta(g)
		& =
		E_\pi \otimes E_\tau \\
		& =
		\left(
		\sum_{I \in [n]^l, J \in [n]^k}
		\delta_{\pi, (I,J)}
		E_{I,J}
		\right)
		\otimes
		\left(
		\sum_{X \in [n]^m, Y \in [n]^q}
		\delta_{\tau, (X,Y)}
		E_{X,Y} 
		\right)
		\\
		& =
		\sum_{(I,X) \in [n]^{l+m}, (J,Y) \in [n]^{k+q}}
		\delta_{\omega, (I,X),(J,Y))}
		E_{(I,X),(J,Y)} \label{deltaprod} \\
		& = E_\omega \\
		& = \Theta(f \otimes g)
	\end{align}
	where (\ref{deltaprod}) holds because $S_\pi((I,J)) \cup S_\tau((X,Y))
					= S_\omega((I,X),(J,Y))$.
				\item It is clear from the statement of the theorem that
					$\Theta$ sends the unit object $0$ in $\mathcal{P}(n)$ to $(\mathbb{R}, 1_\mathbb{R})$, which is the unit object in $\mathcal{C}(S_n)$.
				\item This is immediate because the map (\ref{partmorphismmap}) is $\mathbb{R}$--linear by Theorem \ref{diagbasisSn}. 
		\qedhere
		\end{enumerate} 
	\end{proof}

	\begin{proof}[Proof of Theorem \ref{brauerO(n)functor}]
		The proof is practically identical to that of Theorem \ref{partfunctor}, except we replace $\Theta$ by $\Phi$.
	\end{proof}

	\begin{proof}[Proof of Theorem \ref{brauerSp(n)functor}]
		The proof is very similar to that of Theorem \ref{partfunctor}. 
		For point 2., we need to consider the form of the matrix (\ref{matrixSp(n)}),
		but it is not hard to see how the proof of point 2. for Theorem \ref{partfunctor} can be adapted for this situation, since, 
		if $d_{\beta_1}$ is a $(k,l)$--Brauer diagram, and
		if $d_{\beta_2}$ is a $(q,m)$--Brauer diagram, then,
		defining $\omega \coloneqq \beta_1 \cup \beta_2$, that is,
		$d_\omega$ is a $(k+q,l+m)$--Brauer diagram,
		we have that
		two vertices in $d_{\beta_i}$ are in different rows (the same row) if and only if they are in different rows (the same row) of $d_\omega$.
	\end{proof}
	
	\begin{proof}[Proof of Theorem \ref{brauerSO(n)functor}]
		The proof can be found in \cite[Theorem 6.1]{LehrerZhang2}.
	\end{proof}

\end{appendix}

\end{document}